\documentclass[11pt,letterpaper]{article}
\PassOptionsToPackage{table}{xcolor}
\usepackage{labreport}
\reportlab{}
\reporttype{}
\reportnumber{}
\reportdate{}
\reportshorttitle{ROAD: Modern Adversarial Training for One-Step Distillation}
\reportcolors{183153}{167D8D}
\reportlinkcolor{2455A4}
\usepackage{etoolbox}
\patchcmd{\maketitle}{\vspace{14pt}}{\vspace{11pt}}{}{\PackageError{ROAD}{Title spacing patch failed}{}}
\patchcmd{\maketitle}{\vspace{18pt}}{\vspace{14pt}}{}{\PackageError{ROAD}{Title spacing patch failed}{}}
\pretocmd{\maketitle}{\vspace*{16pt}}{}{\PackageError{ROAD}{Title balance patch failed}{}}
\patchcmd{\abstract}{\fontsize{10.5}{13.5}}{\fontsize{11}{14.5}}{}{\PackageError{ROAD}{Abstract typography patch failed}{}}

\def\plaineqref#1{\ref{#1}}

\def\1{\bm{1}}

\DeclareMathAlphabet{\mathsfit}{\encodingdefault}{\sfdefault}{m}{sl}
\SetMathAlphabet{\mathsfit}{bold}{\encodingdefault}{\sfdefault}{bx}{n}

\usepackage{makecell,placeins,pifont}

\definecolor{roadhl}{RGB}{255,243,205}
\definecolor{revtext}{RGB}{0,0,0}
\definecolor{hlblue}{RGB}{0,0,0}
\definecolor{scopegreen}{RGB}{0,0,0}
\definecolor{darkF7E0D5}{RGB}{209,154,128}
\newcommand{\scopechange}[1]{#1}

\newcommand{\rownumber}[1]{\textcolor{darkF7E0D5}{#1}}
\newcommand{\cmark}{\ding{51}}
\newcommand{\xmark}{\ding{55}}
\newcommand{\mcite}[2]{\hyperlink{cite.#2}{#1}\nocite{#2}}

\title{ROAD: Modern Adversarial Training for Simple One-Step Distillation at Scale}
\author{Xiaoou Li\textsuperscript{1}\qquad
  Shan You\textsuperscript{2}\qquad
  Tao Huang\textsuperscript{1}}
\reportaffiliations{\textsuperscript{1}Shanghai Jiao Tong University\quad(\href{https://sjtu.edu.cn}{sjtu.edu.cn})\\[3pt]
  \textsuperscript{2}SenseTime\quad(\href{https://sensetime.com}{sensetime.com})}
\reportauthornote{Code: \url{https://github.com/Lixiaoou729/ROAD}}

\AtBeginDocument{\hypersetup{pdfauthor={Xiaoou Li, Shan You, Tao Huang}}}

\begin{document}

\maketitle

\begin{abstract}
	One-step distillation turns pretrained diffusion models into fast generators,
	but strong results often require substantial training machinery and compute.
	{\color{hlblue}Pure adversarial distillation offers a simpler route, yet has struggled at
	modern text-to-image scale. To address this limitation, we propose ROAD (Representation-grounded
	One-step Adversarial Distillation), a one-stage framework that distills
	text-to-image models with one adversarial objective and a lightweight
	representation-space critic. First, we replace the classical convolutional
	head with a Lipschitz-controlled Transformer head
	with token-level text conditioning, which matches the semantic tokens of modern visual encoders. Second, we observed a fixed
	random feature mixing leads to long-horizon degradation. We attribute it to
	generator--critic co-adaptation around one realization-specific view, and
	periodically re-draw the mixing to keep the randomization effective.
	Third, for our dense and vector-valued logits, we introduce Jacobian R1 which penalizes the full input–logit Jacobian, overcoming pooled R1's limitation of neglecting most of logit sensitivities.} A single
	adversarial objective then suffices: in one step, ROAD reaches 11.03 FID on
	SDXL and 13.73 on SD3.5-M on COCO-10K, surpassing the multi-step teachers at a
	fraction of distribution matching's training compute.
\end{abstract}

\begingroup
\fontsize{10.5}{12.7}\selectfont
\setlength{\parskip}{0pt}
\reportcontents
\endgroup

\noindent
\begin{minipage}{\linewidth}
\captionsetup{type=figure,hypcap=false}%
\centerline{\resizebox{\linewidth}{!}{%
  \includegraphics[height=2.02in]{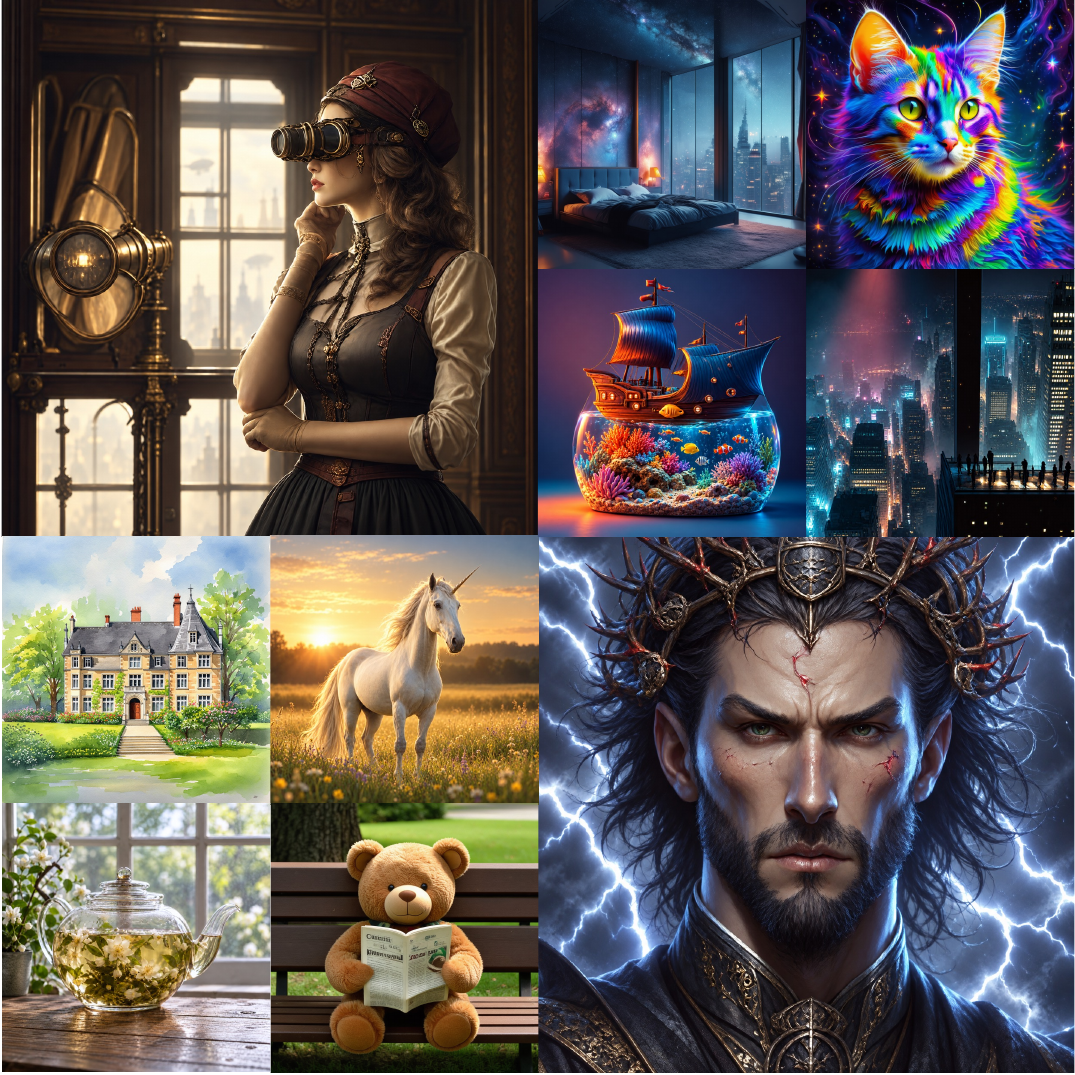}%
  \hspace{0.10in}%
  \includegraphics[height=2.02in]{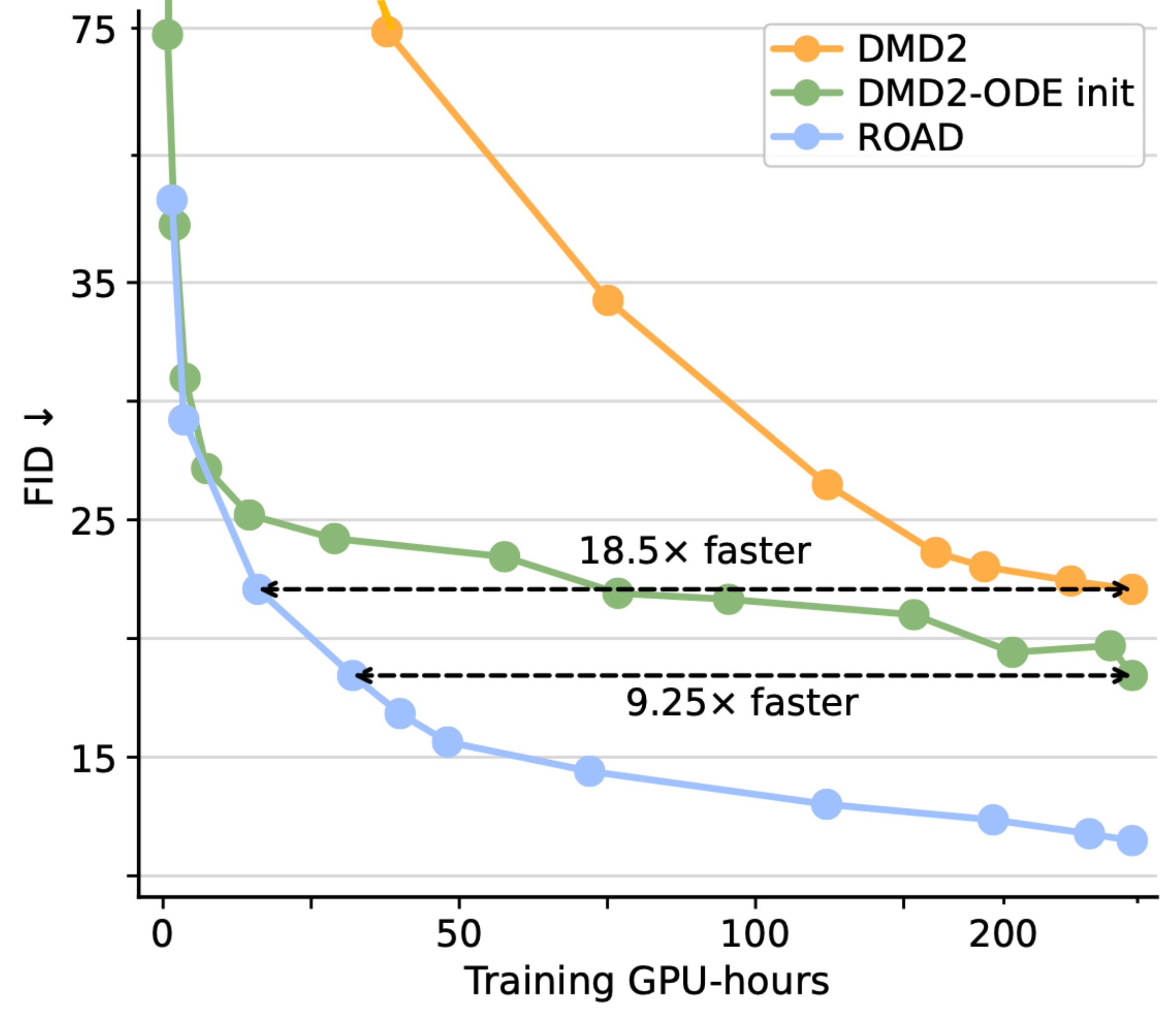}%
}}\par
\vspace{2pt}
\caption{\textbf{A single adversarial loss suffices for one-step distillation.}
Left: samples from ROAD. Right: one-step SDXL FID vs.\ training GPU-hours
(our runs). DMD2 starts from the base model; \mbox{DMD2-ODE}~init
\citep[released recipe;][]{yin2024dmd2} starts with a regression stage on
teacher ODE pairs. ROAD needs no such stage, beats both baselines' best FID
early, and keeps improving.}
\label{fig:teaser}
\end{minipage}

\section{Introduction}
Diffusion and flow models have achieved remarkable quality in text-to-image generation, but their iterative sampling remains expensive at inference time. Step distillation therefore aims to compress a many-step generator into one or a few evaluations. Paradoxically, however, the strongest one-step methods often simplify inference by making training substantially more complicated: distribution matching introduces an online score network at base-model scale~\citep{luo2023diffinstruct,yin2024dmd,nguyen2024swiftbrush,zhou2024sid}, while competitive systems commonly rely on additional objectives, auxiliary models, or multiple stages (Tab.~\ref{tab:sdxl}, Sec.~\ref{app:related}). Thus, \emph{one-step inference has become simple, but one-step training has not}.

Adversarial distillation offers a strikingly simpler alternative. Rather than matching scores or trajectories, it can train a one-step generator directly against real data with a single adversarial objective, without an online score network or repeated teacher queries. Yet this simplicity has not carried over to modern text-to-image models. Pure adversarial training at this scale has struggled with stability and prompt alignment~\citep{sauer2024add,yin2024dmd2}. As a result, adversarial losses are typically used only to assist a heavier distillation objective~\citep{sauer2024add,zhou2024sida}. When adversarial training becomes the main objective, existing methods often recover performance by introducing other sources of complexity, such as teacher-generated supervision, previously distilled generators, or large trainable critics~\citep{sauer2024ladd,chen2025nitrofusion,lin2025apt}. This leaves open a more fundamental question: \emph{is pure adversarial distillation itself unable to scale, or is the critic used to train it the actual bottleneck?}

We investigate the latter possibility. The critic used in adversarial distillation is largely inherited from projected GANs~\citep{sauer2021projected,sauer2023stylegant}, while the representation space it operates on has changed substantially. Modern visual encoders such as DINOv3 and Perception Encoder already map images into contextual, semantically organized token representations~\citep{simeoni2025dinov3,bolya2025pe}. Yet these representations are still read through lightweight convolutional heads; the random feature transform is sampled once and held fixed throughout training; and gradient regularization acts only on a scalar obtained by pooling the dense critic outputs. We find that these seemingly innocuous design choices become increasingly restrictive at modern model scale. They limit how the critic interprets semantic representations, allow its judgments to become increasingly tied to a single feature view during prolonged optimization, and leave individual logit sensitivities insufficiently controlled. This diagnosis leads us to revisit pure adversarial distillation with a modern critic tailored to today's representations and optimization dynamics, making its simple training recipe effective at text-to-image scale. We instantiate this idea in ROAD, \textbf{R}epresentation-grounded \textbf{O}ne-step \textbf{A}dversarial \textbf{D}istillation. 

\textbf{Reading modern representations with modern ViT heads.}  The inherited convolutional readout builds a visual pyramid over fixed neighborhoods, whereas
modern token representations call for a head that selects and relates semantic
evidence by content. We therefore use ViT heads made viable for adversarial
training through Lipschitz controls, while cross-attention to individual prompt tokens improves text alignment. The resulting
critic clearly outperforms similar-sized convolutional heads in fidelity and
alignment (Sec.~\ref{sec:head}).

\textbf{Keeping the random mixing random.} Projected critics read the frozen
features through a random mixing, drawn once and held fixed so that the critic
does not rely on a few dominant feature directions
\citep{sauer2021projected}. This protection, however, does not last over long
horizons: training under a fixed draw converges quickly, then oscillates and
slowly deteriorates. We attribute this to
generator--critic co-adaptation around one persistent view of the
representation: directions measured weakly by that view can remain poorly
corrected. ROAD periodically redraws the mixing while carrying forward the
generator, heads, and optimizer states. This changes the feature view without
restarting learning, limits prolonged co-adaptation, and sustains improvement
where fixed-mixing runs deteriorate (Sec.~\ref{sec:refresh}).

\textbf{Adapting R1 to dense, vector-valued critics.} The classical R1 penalty
regularizes a scalar critic through its input gradient on real images and is a
staple of stable GAN training
\citep{mescheder2018r1,karras2020stylegan2,kang2024diffusion2gan}. ROAD's
critic, however, is dense and vector-valued, and the adversarial loss acts on
every logit, whereas the customary extension pools the logits into a scalar
before applying R1 \citep{zheng2025d2o,suvorov2022lama,kang2023gigagan}.
{\color{hlblue}Pooling keeps only one direction of the critic's output space,
so the input--logit Jacobian's sensitivities in every other direction, however
large, never enter the penalty.}
The proposed Jacobian R1 (JR1) instead penalizes every logit's input gradient,
i.e., the squared Frobenius norm of the input--logit Jacobian. A single unbiased Hutchinson probe makes it practical with one vector–Jacobian product (Sec.~\ref{sec:jr1}).

ROAD proves to be a simple, low-cost, and strong one-step distillation method:
with one loss and one stage, it reaches 11.03 FID on SDXL and 13.73 on SD3.5-M
on COCO-10K. On SDXL, it reaches the best FID of ODE-initialized DMD2 using
1/9.25 of the training GPU-hours (Fig.~\ref{fig:teaser}).

\section{Preliminaries}
\label{sec:prelim}

\noindent\textbf{Pure adversarial one-step distillation.}
{\color{revtext}
	Given a pretrained text-to-image diffusion model, we initialize a one-step
	student $G_\theta$ from its architecture and weights. The student maps Gaussian
	noise $\epsilon\sim\mathcal N(0,I)$ and a prompt $c\sim p_c$ directly to an
	image at a fixed conditioning timestep, without querying the teacher thereafter.
	Following Projected GAN, the critic judges images
	through a frozen pretrained encoder $F$, whose features are tapped at $L=4$
	depths, $f_\ell(x)$ for $\ell=1,\dots,L$. The conditional head at level $\ell$
	emits dense $N_\ell$ per-position logits. Concatenating the logits across all levels
	gives a vector $D_\phi(x,c)=[D_1(x,c),\dots,D_n(x,c)]^\top\in\mathbb R^n$, where
	$n=\sum_\ell N_\ell$. For real image--prompt pairs $(x,c)\sim p_r$, the critic and generator
	minimize, respectively,
	\begin{equation}
		\begin{aligned}
			\mathcal L_D
			&=\mathbb E_{(x,c)\sim p_r}\Big[\sum\nolimits_i\omega_i\,
			\zeta\big(-D_i(x,c)\big)\Big]
			+\mathbb E_{\epsilon,c}\Big[\sum\nolimits_i\omega_i\,
			\zeta\big(D_i(G_\theta(\epsilon,c),c)\big)\Big],\\
			\mathcal L_G
			&=\mathbb E_{\epsilon,c}\Big[\sum\nolimits_i\omega_i\,
			\zeta\big(-D_i(G_\theta(\epsilon,c),c)\big)\Big], 
		\end{aligned}
		\label{eq:dense}
	\end{equation}
	where 	$\omega_i=1/n$ and $\zeta$ is the softplus function.
	These are the logistic critic and non-saturating generator losses
	\citep{goodfellow2014gan}, with the nonlinear loss applied to each logit before
	uniform averaging. We alternate minimizing $\mathcal L_G$ over $\theta$ and
	$\mathcal L_D+R_D$ over $\phi$, where $R_D$ is the critic gradient penalty
	defined below and is zero in unregularized configurations. 
}

\noindent\textbf{Random mixing in projected critics.}
Projected GAN feeds the frozen features
through \textit{non-learned random affine maps} (i.e., CCM and CSM) before the trainable multi-scale heads. CCM reorganizes channels within each feature
level, and CSM further combines information from different
levels along a top-down path. We denote the collection of mixing
parameters by $M$, the critic's measurement basis (subscripts index feature
levels; parenthesized superscripts index draws). Schematically,
\begin{equation}
\tilde f_\ell=M^{\mathrm{cc}}_\ell f_\ell,
\qquad
\hat f_\ell=M^{\mathrm{cs}}_\ell\big(\hat f_{\ell+1}+\tilde f_\ell\big),
\qquad
\hat f_L=M^{\mathrm{cs}}_L\tilde f_L,
\label{eq:mixing}
\end{equation}
 where the top-down recursion applies for $\ell=L-1,\cdots,1$.
In the classical design, these maps are sampled at initialization
and then held fixed. By expressing pretrained features in mixed
coordinates, the heads base their decisions on combinations of
cues across channels and depths. This prevents the discriminator
from relying solely on a few dominant feature directions and
encourages broader feature use \citep{sauer2021projected}.

{\color{hlblue}
\noindent\textbf{Pooled R1 regularization.}
\label{sec:gp}
	For the vector-valued critic of Eq.~(\plaineqref{eq:dense}), the pooled R1 first
	averages the $n$ logits and applies the classical R1 penalty to the resulting
	scalar:
	\begin{equation}
		R_1^{\mathrm{pool}}
		=\frac{\gamma}{2}\,
		\mathbb E_{(x,c)\sim p_r}
		\big\lVert
		\nabla_x\big( \tfrac1n\sum\nolimits_i D_i(x,c)\big)
		\big\rVert_2^2 .
		\label{eq:r1}
	\end{equation}
	 This pooled form controls the input sensitivity of the averaged logit. The common practice is to apply this regularization  lazily to reduce cost \citep{karras2020stylegan2,kang2023gigagan}. 
}

\section{Scaling Pure Adversarial Distillation with a Representation-Grounded Critic}
\label{sec:method}

 ROAD trains the one-step student  with the dense
adversarial losses  and the JR1
penalty against a projected critic assembled from
three parts: the frozen encoder, the random mixings, which
ROAD periodically re-draws, and trainable Transformer heads that
attend to the prompt's text tokens and emit dense per-position
logits (Fig.~\ref{fig:arch}). The teacher appears only as the
student's initialization.

\subsection{A Transformer Head for Semantic Token Representations}
\label{sec:head}

 A natural starting point is to retain the convolutional discriminator design of Projected GAN~\citep{sauer2021projected}. In our one-step distillation setting, this convolutional (Conv) readout  remains trainable but underperforms in both image quality and prompt adherence, and simply increasing its capacity (Deep Conv) further degrades performance, as shown in rows \rownumber{1} and \rownumber{2} of Tab.~\ref{tab:head-ablation}, respectively.

We attribute this to a change in what the trainable head is for. A modern ViT
backbone such as DINOv3 already maps an image into a contextual, semantic
representation space, and every tap returns the same number of tokens. The head
need not rebuild the visual pyramid that projected GAN's convolutional
heads construct; it has only to decide which tokens matter and how they relate. Convolution aggregates over predetermined neighborhoods with
fixed, translation-shared weights, an inductive bias suited to building a
hierarchy from raw pixels rather than to reading backbone features: two tokens
are combined because they are adjacent, not because their content is related.
Self-attention decides instead from the tokens themselves which of them are
combined and how strongly, which makes a Transformer head the more promising
choice.

\begin{figure}[t]
\centering
\includegraphics[
  width=\linewidth
]{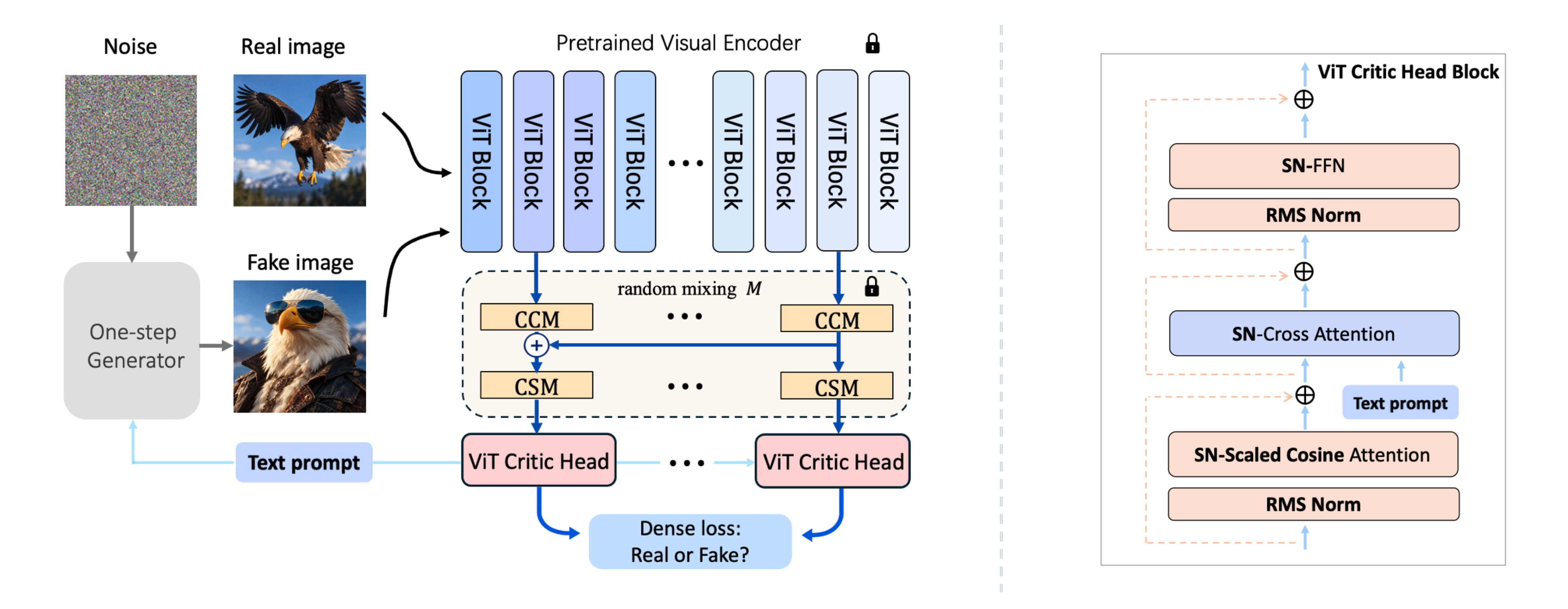}
\caption{\textbf{Overall architecture and critic head's ViT block of ROAD.}
Left: generated and real images are encoded into multi-level semantic
features, transformed by periodically re-drawn random mixings $M$, and
scored by Transformer critic heads. Right: a Lipschitz-controlled
ViT block uses spectral-normalized scaled-cosine self-attention, token-level
cross-attention, and an FFN, together with RMSNorm and learnable residual
scales.}
\label{fig:arch}

\end{figure}

\begin{table}[t]
\caption{\scopechange{\textbf{Critic head design  and text conditioning.} Fixed mixing and pooled R1 are used with 2B generator (SDXL). \textbf{Left:} Conv follows the convolutional head  of Projected GAN. Raising its capacity does not improve performance, whereas our  ViT head substantially improves both fidelity and alignment; the naive Transformer head collapses. \textbf{Right:} Pooled-text projection provides only a minor gain over the unconditional critic, whereas token-level cross-attention yields a  stronger improvement in text--image alignment.}}
\label{tab:head-ablation}
\centering
\scriptsize
\setlength{\tabcolsep}{3pt}
\renewcommand{\arraystretch}{1.05}
\begin{minipage}[t]{0.57\linewidth}
\vspace{0pt}
\centering
\begin{tabular*}{\linewidth}{@{\extracolsep{\fill}}clccc@{}}
\toprule
 & Critic head & Param. & FID $\downarrow$ & CLIP $\uparrow$ \\
\midrule
\rownumber{1} & Conv & 24.8M & 19.02 & 0.317 \\
\rownumber{2} & Deep Conv & 54.8M & 21.52 & 0.311 \\
\rownumber{3} & Naive Transformer & 54.4M & Collapse & -- \\
\rownumber{4} & ViTGAN & 54.3M & 15.40 & 0.319 \\
\rownumber{5} & ViT (ours) & 54.3M & 14.54 & 0.331 \\
\bottomrule
\end{tabular*}
\end{minipage}\hfill
\begin{minipage}[t]{0.40\linewidth}
\vspace{0pt}
\centering
\begin{tabular*}{\linewidth}{@{\extracolsep{\fill}}clc@{}}
\toprule
 & \scopechange{Text conditioning} & CLIP $\uparrow$ \\
\midrule
\rownumber{6} & Unconditional & 0.320 \\
\rownumber{7} & \makecell[l]{Miyato-style\\pooled-text projection} & 0.322 \\
\rownumber{8} & Token-level cross-attention & 0.331 \\
\bottomrule
\end{tabular*}
\end{minipage}
\end{table}

Lipschitz control, however, is necessary to make this replacement
viable in the adversarial game~\citep{zhou2019lipschitz,miyato2018sn}. It limits how strongly the critic can
amplify changes in its input and thereby constrains abrupt variation
in the signal returned to the generator. A naive Transformer head
collapses in our setting (Tab.~\ref{tab:head-ablation}, row \rownumber{3}). This
failure is consistent with a known difficulty: standard dot-product
softmax attention is not globally Lipschitz in its inputs, even when
the learned linear maps have bounded spectral norm
\citep{kim2021lipschitz}.

\textbf{A Lipschitz-controlled ViT head} is therefore introduced, which
constrains the gain of every trainable component through three complementary
mechanisms. First,
\textit{spectral normalization} \citep{miyato2018sn} is applied to
every learned linear map in attention, the FFN, and the logit readout,
so that its spectral norm is no more than 1. Second, \textit{cosine attention
	with a clamped temperature} normalizes every query
and key before their comparison  under
a learnable per-head temperature clamped at $s_{\max}$ \citep{liu2022swinv2}. This  removes
their raw activation magnitudes from the attention logits and bounds
the remaining scale explicitly. Third, \textit{each residual branch uses a
	learnable per-channel scale} initialized to $1/B_\ell$ for a
level-$\ell$ head of depth $B_\ell$, preventing the cumulative residual gain from growing excessively with the
head depth. On any convex region where queries and keys keep a norm floor before
$\ell_2$-normalization, these controls make the head's map from mixed tokens to
logits Lipschitz with the explicit constant $\Lambda_\ell$ of
Eq.~(\plaineqref{eq:lip}), proved as Lemma~\ref{lem:lip} in
Sec.~\ref{app:lipschitz}.

With the controls in place, the ViT head trains successfully and substantially improves the quality of generated images (Tab.~\ref{tab:head-ablation}, row \rownumber{5}). ViTGAN~\citep{lee2022vitgan} provides an alternative ViT discriminator design, but remains inferior to our head under the same protocol (Tab.~\ref{tab:head-ablation}, row \rownumber{4}).
The weak prompt adherence long associated with adversarial
distillation, however, persists \citep{sauer2024add,yin2024dmd2}.
Comparing text-conditioning schemes, we attribute this shortfall to
the pooled prompt vector: a single pooled embedding cannot express
fine-grained linguistic structure.

\noindent\textbf{Token-level text conditioning via cross-attention.}
The Transformer head admits token-level cross-attention naturally,
and cross-attention provides a direct mechanism for image tokens at
different spatial positions to interact differently with individual
text tokens. We insert cosine cross-attention into the first block of
every head, with subsequent blocks refining the conditioned
representation through self-attention and FFN layer. Each level
keeps its own text key/value projections, allowing different semantic
depths to read the same prompt differently. Under the unified protocol, Miyato-style pooled-text projection provides only a minor improvement over the unconditional critic, whereas token-level cross-attention yields a much stronger gain in CLIP score (Tab.~\ref{tab:head-ablation}, rows \rownumber{6}-\rownumber{8}).

\noindent\textbf{\color{hlblue}Multi-branch random mixing.}
\label{sec:random-mixing}
{\color{hlblue}
	Projected GAN's random mixing is tailored to a convolutional pyramid, where
	CCM is a $1{\times}1$ convolution within a level and CSM adds
	$3{\times}3$ convolutions between levels of differing resolution. A ViT backbone has no such pyramid, since
	all taps share one uniform patch-token grid, so both maps reduce to a plain
	$1{\times}1$ convolution that mixes channels alone. We therefore replace
	each single-kernel mixing by an average of three size-preserving random
	convolutions over the patch tokens, with kernels $1{\times}1$,
	$3{\times}3$, and $5{\times}5$; the summary token is transformed separately
	by a random linear map. The wider branches let a token's mixed value draw on
	its spatial neighbors, so evidence is spread across positions as well as
	channels and depths. Sec.~\ref{app:projection} shows the advantage of this
	design. These frozen multi-branch maps constitute the mixing
	$M$ that Sec.~\ref{sec:refresh} re-draws, and Sec.~\ref{app:mixing-op}
	gives its matrix form and analyzes its anisotropy.}

\subsection{Periodic Mixing Refresh Prevents Critic Co-adaptation}
\label{sec:refresh}

{\color{hlblue}With the Transformer head of Sec.~\ref{sec:head}, training converges quickly
to good performance, then begins to oscillate and slowly deteriorate
(Fig.~\ref{fig:fixed-diagnosis}a). The symptom resembles discriminator
overfitting, yet standard augmentation remedies such as
DiffAugment~\citep{zhao2020diffaugment} and ADA~\citep{karras2020ada}
do not materially alleviate this late-stage degradation under fixed mixing
(Sec.~\ref{app:augmentation}). We therefore examine other possible causes and finally focus on the random mixing, held fixed throughout training in the
classical design. We argue that its anisotropy persistently biases
long-horizon optimization toward the same feature directions. The generator and the critic
then form a \emph{realization-specific co-adaptation} around this single view,
which eventually erodes generation quality.}

\noindent\textbf{A persistent critic state.}
We diagnose the critic with two measurements,
Pre-projection Feature Utilization Coverage (PFUC) and the logit margin (Sec.~\ref{app:diagnostics}). \textit{PFUC} is the
first-order sensitivity of the critic's mean logit to each backbone channel
before mixing, expressed in backbone coordinates. It settles after an early
phase of training and barely moves thereafter
(Fig.~\ref{fig:fixed-diagnosis}b): the TV between consecutive profiles decays
to a floor (Fig.~\ref{fig:fixed-diagnosis}c). The \textit{logit margin} is the
batch-mean real logit minus the batch-mean generated logit, measuring how
confidently the critic separates the two populations. After the first few
thousand steps, the margin stays large and nearly stationary
(Fig.~\ref{fig:fixed-diagnosis}d), even though FID has already begun its usual
oscillation and deterioration. Because the generator is still changing, a
critic decision that has stopped moving is not by itself a healthy sign.
Together, the panels show a persistent critic state, i.e., hardened per-channel
sensitivity profile and stationary logit margin.

\noindent\textbf{Coupling to the fixed mixing.}
Write the critic as $D_\phi\big(M^{(0)}f(x)\big)$, where $M^{(0)}$ is the
training draw. Across training, $D_\phi$ supplies the generator with an
optimization signal shaped by this one realization. Gradually $M^{(0)}$ is absorbed into the critic's
parameters and becomes the fixed geometry within which the adversarial game
continues to be optimized. Finally, the way the signal uses the features
itself hardens, producing the phenomena of Fig.~\ref{fig:fixed-diagnosis}b--d.
\emph{We test this account with two post-hoc experiments: fresh-draw transfer and a
linear probe}. First, we hold one checkpoint fixed and replace only $M^{(0)}$ by
independently drawn mixings. Across  six checkpoints from $5$k to $30$k
steps, the logit margin is about $5.0$ under $M^{(0)}$ but about $0$ under a fresh draw, indicating the critic's real--generated separation  does not transfer to a fresh measurement basis. Second, at each level, we fit a
separate linear classifier to the frozen mixed features to distinguish real
from generated samples. From $2.5$k steps on, the two populations are harder to
separate under $M^{(0)}$ than under a fresh draw, particularly at the shallow
levels. This suggests that what still distinguishes generated from real images
is least visible under $M^{(0)}$, the view through which the generator was trained.
The two diagnostics together show that the critic's
separation is specific to the pairing $D_\phi\circ M^{(0)}$.

\begin{figure}[t]
	\centering
	\includegraphics[width=\linewidth]{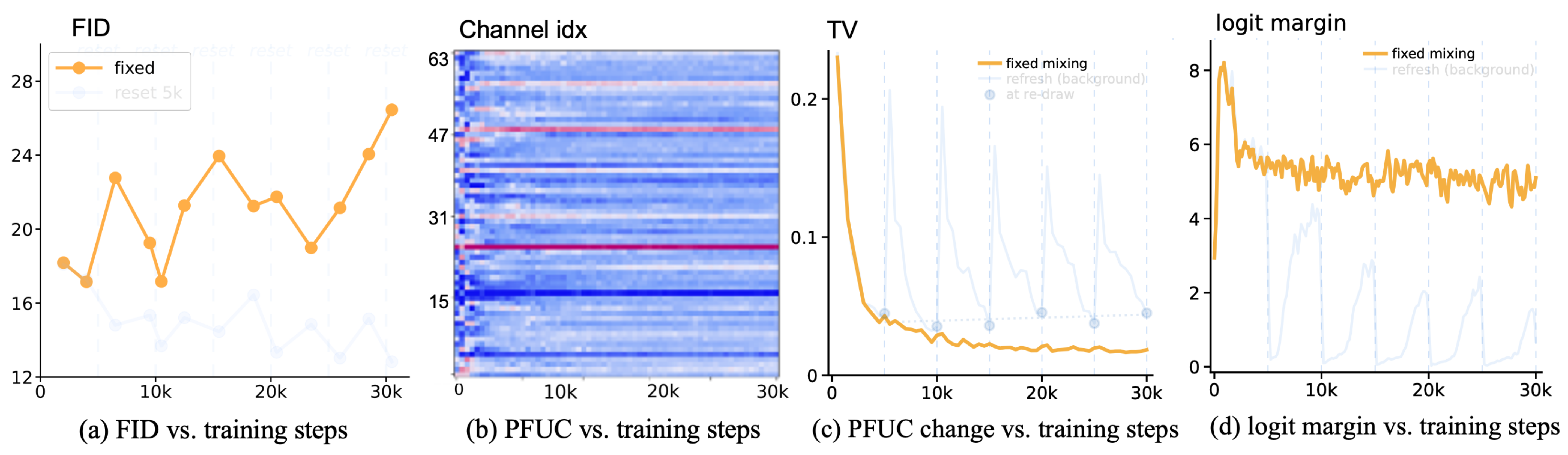}
	\caption{\textbf{Diagnosing co-adaptation under fixed random mixing.}
		(a)~FID oscillates and eventually
		deteriorates. (b)~PFUC settles into a persistent per-channel pattern.
		(c)~The change between consecutive PFUC profiles, measured by total
		variation (TV), decays toward a floor. (d)~The logit margin stays high and
		nearly stationary even as FID degrades. The twin run refreshed every $5$k
		steps is faded for reference. Sec.~\ref{app:diagnostics} describes
		the measurements in detail.}
	\label{fig:fixed-diagnosis}
\end{figure}

\noindent\textbf{Anisotropy-based co-adaptation ends in degradation.}
Stack the frozen features as $f$ and write the affine mixing as $\hat f=Mf+b$ (Sec.~\ref{app:mixing-op}).
As a random $M$ is anisotropic (Sec.~\ref{app:mixing-op}),  the gain $\lambda_M(\eta)=\lVert M\eta\rVert$ that
it applies to a unit feature direction $\eta$ varies widely from one direction
to another. This gain decides both how much the critic's logits change when the
features move along a direction and how much of the generator's correction
reaches it.  
Specifically, the directional derivatives of the critic and the generator along
$\eta$ are bounded as
$|\langle\eta,\nabla_f D_i\rangle|\!\le\!\lambda_M(\eta)\Lambda_\ell$ and
$|\langle\eta,\nabla_f\mathcal L_G\rangle|\!\le\!\lambda_M(\eta)\lVert\nabla_{\hat f}\mathcal L_G\rVert$,
where $\Lambda_\ell$ is the Lipschitz bound of Lemma~\ref{lem:lip} for the
level-$\ell$ head that produces $D_i$.
  So a small-gain direction moves the logits little and is slowly corrected in the generator, however
the critic behaves; we call it \emph{weakly measured}. 
The minimal measured-gap surrogate in Prop.~\ref{prop:gain-correction}(i)  shows the feature-space residual between generated and real data decays along a right-singular direction at a rate equal to the square of that direction's gain.  In light of the proposition, high-gain directions are corrected early, whereas weakly measured directions persist, concentrating the surviving residual in them. 
We thus hypothesize that an unchanging gain landscape sustains
realization-specific co-adaptation.

\noindent\textbf{Periodic mixing refresh suppresses co-adaptation-induced degradation.}
{\color{revtext}
	Given a refresh interval $\tau$, at the start of the $k$-th period we re-draw
	the mixing $M^{(k)}$ from its initialization law. The generator, the critic
	heads, and the optimizer states remain unchanged, and no learnable parameter
	is reset or added. Each re-draw replaces the draw-specific gain landscape, so
	a residual direction that is weakly measured under $M^{(k)}$ may receive a
	higher gain under $M^{(k+1)}$. Under the right-isotropy assumption, independent
	re-draws prevent any fixed direction from retaining a privileged slow
	correction rate (Prop.~\ref{prop:gain-correction}(ii)). After a
	refresh, the critic heads adapted to the preceding mixing are therefore
	temporarily mismatched with the new measurement geometry. Accordingly, PFUC
	reorganizes at each re-draw (Fig.~\ref{fig:refresh-mechanism}b,c), while the
	logit margin resets toward zero and then re-forms as the heads adapt
	(Fig.~\ref{fig:refresh-mechanism}d). By shortening how long the
	generator--critic pair can co-adapt to one fixed view, periodic refresh
	interrupts realization-specific co-adaptation and can re-expose weak residual
	directions to adversarial correction. Consistent with this mechanism, the
	refreshed run re-enters an improving trajectory after re-draws and maintains
	a more stable and lower FID over long-horizon training
	(Fig.~\ref{fig:refresh-mechanism}a).
}

\begin{figure}[t]
	\centering
	\includegraphics[width=\linewidth]{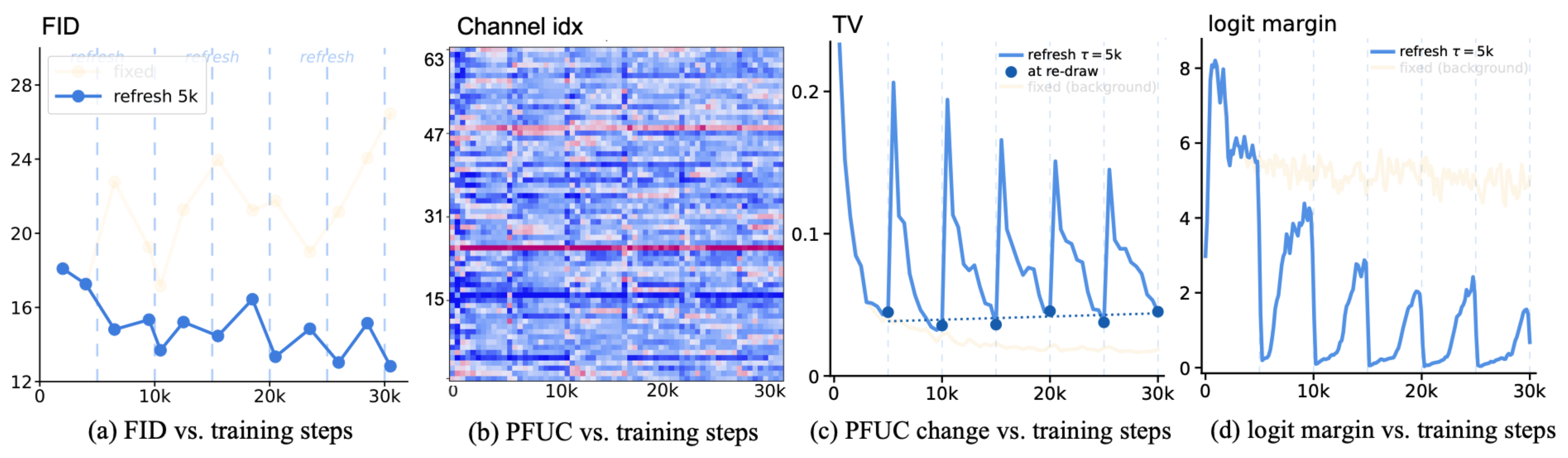}
	\caption{\textbf{Periodic mixing refresh breaks the persistent
			critic state.} The refreshed run (blue, every $5$k steps) is
		foregrounded, with fixed mixing faded for reference. (a)~FID improves
		after re-draws and avoids the late deterioration of fixed mixing.
		(b)~PFUC reorganizes under each new view. (c)~The PFUC change spikes after each re-draw and then decays; filled dots
		mark the value at each re-draw. (d)~The logit margin resets toward zero at each
		re-draw and re-forms as the heads adapt.}
	\label{fig:refresh-mechanism}
\end{figure}

\subsection{Jacobian R1 Regularization for Dense Critics}
\label{sec:jr1}

The objective in Eq.~(\plaineqref{eq:dense}) defines per-logit nonlinear loss across  positions and levels. So its regularizer
should  control the sensitivities of these individual
decisions, not only their aggregate. We abbreviate the vector-valued logit $D_\phi(x,c)$ as $D(x)$ as both $\phi$ and $c$ are fixed for gradient regularization. Let
$J(x)\!=\!\partial D(x)/\partial x\in\mathbb R^{n\times d_x}$  where $d_x$ is input dimension, with the  
$i$-th row being $g_i^\top\!=\!(\nabla_xD_i)^\top$.

\noindent\textbf{Pooled R1 misses logit sensitivities.}
To facilitate analysis, we write the pooled R1 as
\begin{equation}
	\textstyle
	R_1^{\mathrm{pool}}
	=\frac{\gamma}{2}\,\mathbb E_{x\sim p_r}
	\frac1{n^2}\lVert J(x)^\top\mathbf1\rVert_2^2
	=\frac{\gamma}{2}\,\mathbb E_{x\sim p_r}
	\bigl\lVert\tfrac1n\textstyle\sum_{i=1}^n g_i(x)\bigr\rVert_2^2, 
	\label{eq:jr1-pool}
\end{equation}
which penalizes the squared norm of the averaged gradient. The per-image energy
$q_{\mathrm{pool}}(x)\!=\!\tfrac1{n^2}\lVert J(x)^\top\mathbf1\rVert_2^2=\tfrac1{n^2}\mathrm{tr}(J^\top\mathbf1\mathbf1^\top J)$ shows that  pooled R1  
only regularizes the direction spanned by the all-ones vector in the
$n$-dimensional output space, while leaving the remaining $n\!-\!1$
directions  unpenalized. Particularly, for $J^\top\mathbf1=0$ the averaged logit is insensitive to every input perturbation to first order,
while the individual logits need not be. In this case, pooled R1 does not
impose any regularization at all, even though the dense loss of Eq.~(\plaineqref{eq:dense}) is
driven by those individual logits.

\noindent\textbf{JR1 exploits full logit sensitivities.}
{\color{hlblue}Motivated by input--output Jacobian sensitivity
\citep{novak2018sensitivity} and logit-level gradient penalties
\citep{varga2018gradient},} we penalize
the full input--logit Jacobian:
\begin{equation}
	\textstyle
	R_1^{\mathrm J}
	=\frac{\gamma_{\mathrm J}}{2}\,\mathbb E_{x\sim p_r}
	\frac1n\lVert J(x)\rVert_F^2
	=\frac{\gamma_{\mathrm J}}{2}\,\mathbb E_{x\sim p_r}
	\tfrac1n\textstyle\sum_{i=1}^n\lVert g_i(x)\rVert_2^2.
	\label{eq:jr1}
\end{equation}
So JR1 penalizes the mean of the individual
squared norms rather than the squared norm of the mean. In light of the per-image energy
$q_{\mathrm{Jac}}(x)=\tfrac1n\lVert J(x)\rVert_F^2$, all $n$ directions of output space are regularized, including the $n\!-\!1$
that pooled R1 neglects. Particularly, when $J^\top\mathbf1=0$, pooled R1 imposes no regularization while JR1 retains every individual
sensitivity. \emph{Prop.~\ref{prop:pooling} provides detailed theoretical comparison between pooled R1 and JR1.}

\noindent\textbf{Efficient estimation with Hutchinson probes.}
Exact JR1 requires backpropagating each of the $n$ logits separately, i.e., one vector--Jacobian product (VJP), which is computationally expensive.
We instead use the classical Hutchinson trace estimator \citep{hutchinson1989},
following its use for Jacobian regularization
\citep{varga2018gradient,hoffman2019jacobian}. Let $v\in\{-1,+1\}^{n}$ have
independent, equiprobable signs, so that $\mathbb{E}[vv^{\top}]=I_n$.
As
$\mathbb{E}_v\|J^{\top}v\|_2^{2}
=\operatorname{tr}(J^{\top}\mathbb{E}[vv^{\top}]J)=\|J\|_F^{2}$, we have
\begin{equation}
	R^{\mathrm J}_1
	=\frac{\gamma_{\mathrm J}}{2n}\,\mathbb{E}_{x\sim p_r}\,\mathbb{E}_v
	\bigl\|J(x)^{\top}v\bigr\|_2^{2}.
	\label{eq:jr1-hutchinson}
\end{equation}
In practice we 
approximate $\mathbb{E}_v$ by a Monte-Carlo average over $P$ independent probes, and use $P{=}1$ per real image, drawn afresh on the lazy
schedule of JR1 critic updates. \emph{Prop.~\ref{prop:jr1} shows that the estimate is unbiased and provides its variance.}

\subsection{The complete recipe.}
\label{sec:recipe}

 So far, the three components accomplish modernization of projected GAN, addressing three degrees of freedom of a representation-space
critic, what judges (Sec.~\ref{sec:head}), through what measurement
(Sec.~\ref{sec:refresh}), and how sharply (Sec.~\ref{sec:jr1}). Together
they make the minimal objective of Eq.~(\plaineqref{eq:dense})
sufficient at modern text-to-image scale.

\section{Experiments}
\label{sec:experiments}

\noindent\textbf{Setup.}
We train one-step generators for SDXL~\citep{podell2024sdxl}, a diffusion model with a UNet architecture, and SD3.5-M, a flow-matching model with an MMDiT architecture. Across the paper, DINOv3 ViT-B is chosen as the backbone for feature extraction unless otherwise specified. We use a dataset of 500K images from LAION-Aesthetic 5.5+~\citep{schuhmann2022laion5b} built by DMD2. Following~\citep{yin2024dmd2,sauer2024add,zhou2024sidlsg}, we evaluate on a random set of 10K COCO-2014 validation prompts. For both models we report FID, CLIP score, and aesthetic score. We additionally report human-preference metrics: PickScore~\citep{kirstain2023pickscore} and HPSv2~\citep{wu2023hpsv2} for SDXL, as in~\citep{lu2025adm,ren2024hypersd}, and PickScore, HPSv2, ImageReward~\citep{xu2023imagereward}, and GenEval~\citep{ghosh2023geneval} for SD3.5-M, following~\citep{swd2025,chen2025nitrofusion}. Complete descriptions of the dataset and evaluation metrics, training hyperparameters, and additional experiments are provided in Secs.~\ref{app:protocol}, \ref{app:hparams}, and~\ref{app:projection}, respectively.

\subsection{Main Results}
\label{sec:exp-main}

\begin{table}[t]
\centering
\caption{\textbf{Evaluation of one-step SDXL.} All results are reproduced under our protocol with official checkpoints,
	except DMDX, whose values are from its paper. Unlike prior methods that require multiple objectives and multi-stage training (Sec.~\ref{app:related}), ROAD uses a single stage with a pure GAN loss.
	}
\label{tab:sdxl}
% Table 3 renders at 9.9pt after its existing resizebox.
\fontsize{9.9}{12.1}\selectfont
\setlength{\tabcolsep}{4pt}
\renewcommand{\arraystretch}{1.12}
\begin{tabular*}{\linewidth}{@{\extracolsep{\fill}}llcccccc@{}}
\toprule
\multicolumn{1}{c}{\multirow{2}{*}[-2.5pt]{Method}} & \multicolumn{2}{c}{Training} & \multirow{2}{*}[-2.5pt]{FID $\downarrow$} & \multirow{2}{*}[-2.5pt]{CLIP $\uparrow$} & \multirow{2}{*}[-2.5pt]{HPSv2 $\uparrow$} & \multirow{2}{*}[-2.5pt]{Pick $\uparrow$} & \multirow{2}{*}[-2.5pt]{Aes $\uparrow$} \\
\cmidrule(lr){2-3}
 & \multicolumn{1}{c}{Tech.} & Stage & & & & & \\
\midrule
\mcite{SDXL base}{podell2024sdxl} (cfg 6) & -- & -- & 19.29 & 0.331 & 28.33 & 22.11 & 6.19 \\
\midrule
\mcite{Hyper-SDXL}{ren2024hypersd} & \makecell[l]{Traj.+Reward+\\ScoreD+GAN} & 6 & 33.66 & 0.328 & \textbf{29.97} & \textbf{22.71} & 6.59 \\
\mcite{DMD2}{yin2024dmd2} & ScoreD+GAN & 2 & 19.01 & 0.336 & 28.36 & 22.09 & 6.09 \\
\mcite{DMDX}{lu2025adm} & Adv.\ ScoreD & 2 & -- & \textbf{0.352} & 27.70 & 22.27 & -- \\
\mcite{SiDa}{zhou2024sida} & ScoreD+GAN & 1 & 15.23 & 0.326 & 27.80 & 21.72 & 6.16 \\
\mcite{SDXL-Turbo}{sauer2024add} & GAN+ScoreD & 1 & 24.56 & 0.337 & 28.53 & 22.40 & 6.08 \\
\mcite{SDXL-Lightning}{lin2024sdxllightning} & Traj.+GAN & 5 & 23.93 & 0.316 & 28.13 & 21.92 & 6.19 \\
\mcite{NitroSD-Realism}{chen2025nitrofusion} & ScoreD+GAN & 3 & 20.80 & 0.335 & 28.84 & 22.37 & 6.31 \\
ROAD (ours) & \textbf{GAN} & \textbf{1} & \textbf{11.03} & 0.338 & 28.58 & 22.13 & \textbf{6.67} \\
\bottomrule
\end{tabular*}
\end{table}

\begin{table}[t]
\centering
\caption{\textbf{Evaluation of few-step SD3.5-M.} All results are reproduced under our
	protocol with official 4-step checkpoints; there are no 1-step methods yet, to the best of our knowledge.}
\label{tab:sd35}
\footnotesize
\setlength{\tabcolsep}{4pt}
\renewcommand{\arraystretch}{1.12}
\resizebox{\linewidth}{!}{%
\begin{tabular}{llcccccccc}
\toprule
Method & Tech. & NFE & FID $\downarrow$ & CLIP $\uparrow$ & HPSv2 $\uparrow$ & IR $\uparrow$ & Pick $\uparrow$ & Aes $\uparrow$ & GE $\uparrow$ \\
\midrule
SD3.5-M (cfg 4.5) & -- & 40 & 18.01 & 0.341 & 29.19 & 1.00 & 22.38 & 6.12 & 0.69 \\
\midrule
SD3.5-M-Turbo & -- & 4 & 23.22 & 0.330 & 26.80 & 0.47 & 21.59 & 5.86 & 0.55 \\
\mcite{SenseFlow}{ge2026senseflow} & \makecell[l]{ScoreD+GAN+Traj.} & 4 & 28.07 & 0.328 & \textbf{29.54} & 0.93 & 22.59 & 6.28 & 0.62 \\
\mcite{SwD}{swd2025} & \makecell[l]{ScoreD+GAN+MMD} & 4 & 26.56 & 0.338 & 29.48 & \textbf{1.15} & \textbf{22.91} & \textbf{6.39} & \textbf{0.71} \\
\mcite{SiD-DiT}{zhou2025scoreflow} & ScoreD & 4 & 22.02 & 0.338 & 29.36 & 1.01 & 22.59 & 6.24 & 0.64 \\
ROAD (ours) & \textbf{GAN} & 1 & \textbf{13.73} & \textbf{0.342} & 28.85 & 0.94 & 22.00 & 6.15 & 0.66 \\
\bottomrule
\end{tabular}}
\end{table}

	\textit{Does a single GAN loss reach state-of-the-art one-step quality?} \textbf{On SDXL}, we compare with strong baselines in Tab.~\ref{tab:sdxl}. Typically, most of these methods consist of multiple training losses and require trajectory (traj.) distillation, score distillation (ScoreD), GAN, and reward learning, along with multiple stages such as ODE regression training followed by ScoreD+GAN. Compared to them,  ROAD is much simpler, with a  pure GAN loss and  a single  training stage. Nevertheless, it achieves 11.03 FID, ahead of all the baselines, and is also competitive on CLIP score, aesthetic score, and preference scores, suggesting that ROAD does not trade preference or alignment for fidelity. The comparison \textbf{on SD3.5-M}, a flow-based MMDiT generator, is given in Tab.~\ref{tab:sd35}. To the best of our knowledge, no one-step generators exist for it yet, so we compare with few-step models. Again, one-step ROAD achieves strong results on FID, much better than the competing methods as well as the teacher. Its CLIP score is among the best, indicating promising text-image alignment, while its preference scores are slightly inferior to those of the compared methods. We note that all these SOTA methods use 4 steps against our single step and are therefore 4 times slower than ours, while most adopt three training losses against our single GAN loss.

\subsection{Ablation Study}
\label{sec:exp-abl}

Across this section, the one-step SDXL generator is used for ablation, in which the refresh interval $\tau$ is 5k and the regularization is JR1 with $\gamma_{\mathrm J}=7.5\times10^{-3}$ unless otherwise specified.

\begin{figure}[t]
\centering
\begingroup
\captionsetup{type=table}
\caption{\textbf{Analysis of key components of ROAD.}}
\label{tab:ablation}
\renewcommand{\arraystretch}{0.94}
\setlength{\tabcolsep}{3pt}
\newcommand{\subgap}{\hspace{13pt}}
{\scriptsize
\begin{tabular}[t]{@{}lccccc@{}}
\multicolumn{6}{@{}c@{}}{\small (a) Component decomposition}\\[3pt]
\toprule
Critic & Lip. Ctrl. & Refresh & Regularizer & FID $\downarrow$ & CLIP $\uparrow$ \\
\midrule
ViT & \xmark & \xmark & Pooled R1 & collap & --    \\
ViT & \cmark & \xmark & Pooled R1 & 14.54  & 0.331 \\
ViT & \cmark & 5k     & Pooled R1 & 12.08  & 0.335 \\
ViT & \cmark & \xmark & JR1 & 13.62  & 0.332 \\
ViT & \cmark & 5k     & JR1                 & \textbf{11.03} & \textbf{0.338} \\
\bottomrule
\end{tabular}\subgap
\color{black}%
\begin{tabular}[t]{@{}lcc@{}}
\multicolumn{3}{@{}c@{}}{\small (b) Frozen backbone}\\[3pt]
\toprule
Frozen backbone & FID $\downarrow$ & CLIP $\uparrow$ \\
\midrule
\mcite{PE-Spatial ViT-B}{bolya2025pe} & 13.59 & 0.332 \\
\mcite{SigLIP2 ViT-B}{tschannen2025siglip2} & 13.37 & 0.334 \\
\mcite{DINOv3 ViT-S}{simeoni2025dinov3} & 15.14 & 0.328 \\
DINOv3 ViT-B & 12.08 & 0.335 \\
DINOv3 ViT-L & 16.33 & 0.325 \\
\bottomrule
\end{tabular}\subgap
\begin{tabular}[t]{@{}lcc@{}}
\multicolumn{3}{@{}c@{}}{\small (c) Regularizer}\\[3pt]
\toprule
Regularizer & FID $\downarrow$ & CLIP $\uparrow$ \\
\midrule
\mcite{Pooled R1}{zheng2025d2o} & 12.08 & 0.335 \\
\mcite{Feature R1}{sauer2024add} & 13.65 & 0.327 \\
\mcite{Approx. R1}{lin2025apt} & 14.49 & 0.330 \\
JR1 (ours) & \textbf{11.03} & \textbf{0.338} \\
\bottomrule
\end{tabular}%
}

\endgroup
\vspace{10pt}
\begin{minipage}[t]{0.49\linewidth}
  \centering
  \includegraphics[width=0.88\linewidth]{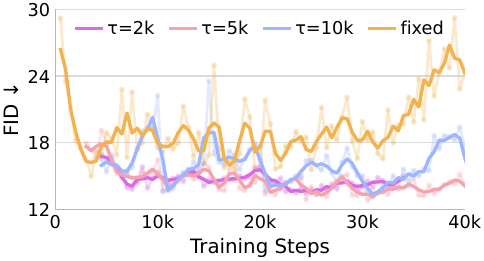}
  \captionof{figure}{\textbf{Effect of refresh interval without gradient regularization.}}
  \label{fig:reset-interval}
\end{minipage}\hfill
\begin{minipage}[t]{0.49\linewidth}
  \centering
  \includegraphics[width=0.88\linewidth]{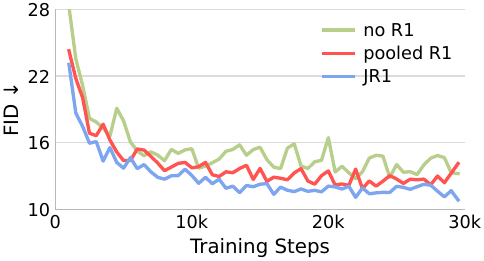}
  \captionof{figure}{\textbf{Pooled R1 against Jacobian R1 (JR1).}}
  \label{fig:regularizers}
\end{minipage}
\end{figure}

\noindent\textbf{Transformer head.}
Tab.~\ref{tab:head-ablation} (Sec.~\ref{sec:head}) isolates the critic head design and the text conditioning under fixed mixing and pooled R1. Increasing the capacity of the convolutional head does not improve performance. Our Lipschitz-controlled ViT head instead yields a substantial gain in both fidelity and alignment, and it also outperforms the ViTGAN alternative. For conditioning, token-level cross-attention provides a clearly stronger alignment gain than Miyato-style pooled-text projection.

\needspace{5\baselineskip}
\noindent\textbf{Component decomposition.}
Tab.~\ref{tab:ablation}(a) presents a breakdown of ROAD’s three  components. Lipschitz control first makes the ViT head trainable, indicating its importance for stabilizing GAN. Periodic refresh then improves FID from 14.54 to 12.08 and CLIP from 0.331 to 0.335, demonstrating its  benefit of refresh mixing over no refresh (i.e. fixed mixing). Finally, three components together reduces FID to 11.03 while improving CLIP to 0.338. Further, trajectory-level comparison (Fig.~\ref{fig:regularizers}) shows  JR1's benefit persists across training rather than appearing only at the selected checkpoint.

\noindent\textbf{Frozen backbone.}
Tab.~\ref{tab:ablation}(b) contrasts feature extractors from several different foundation models.
Among DINOv3 ViTs of different capacities, ViT-B is the strongest, while both the
smaller ViT-S and the larger ViT-L underperform it. This suggests that a stronger
representation does not necessarily translate into better GAN performance.
PE-Spatial ViT-B and SigLIP2 ViT-B are close behind DINOv3 ViT-B, so ROAD is not
tied to one feature extractor.

\noindent\textbf{Regularizer.}
\label{sec:exp-r1}
 Tab.~\ref{tab:ablation}(c) compares JR1 with pooled R1, feature R1 that moves
 R1 penalty from image space to the frozen feature taps
\citep{sauer2024add,chong2024p2d}, and Approx.\ R1 (APT) that approximates the
image-space penalty without a double backward pass \citep{lin2025apt}. JR1 is the
strongest of the four. It improves on pooled R1 from 12.08 to 11.03
FID, whereas both the feature-space and the approximate penalty fall behind pooled
R1. Sec.~\ref{app:projection} further shows that JR1 consistently
	outperforms pooled R1 across a broad range of $\gamma_{\mathrm J}$.

\noindent\textbf{Trajectory of  Refresh interval.}
Fig.~\ref{fig:reset-interval} analyzes the effect of  periodic 
refresh at $\tau\!\in\!\{2\mathrm{k},5\mathrm{k},10\mathrm{k}\}$, without gradient regularization to isolate the refresh effect. Fixed mixing oscillates strongly
and deteriorates late in training, whereas all refreshed variants remain
markedly more stable and achieve better FID. More frequent refresh yields smoother trajectories, consistent with periodic re-drawing interrupting co-adaptation to a single random view. The best FIDs nevertheless remain close across intervals, indicating low sensitivity to 
 $\tau$. We use $\tau=5\mathrm{k}$ by default.

\noindent\textbf{Trajectory of Jacobian R1.}
Fig.~\ref{fig:regularizers} contrasts  no gradient regularization, pooled R1, and JR1 with refresh interval $\tau\!=\!5\mathrm{k}$. Pooled R1 yields a smoother FID trajectory
than the unregularized run, and JR1 maintains a consistently lower FID with smaller
fluctuations across training.  This suggests that gradient regularization  is complementary to our refresh mechanism; besides, the optimization benefits more from JR1 which regularizes the  full logit sensitivities.

\FloatBarrier

\needspace{8\baselineskip}
\subsection{Further Analysis of Mixing, Regularization, and Augmentation}
\label{app:projection}

We further examine the mixing design, sensitivity to the JR1 coefficient, and
augmentation controls for the fixed-mixing degradation.

% Center content vertically in the first row; align captions in the second.
\noindent
\begin{tabular}{@{}c@{\hspace{0.04\linewidth}}c@{}}
\begin{minipage}[c]{0.48\linewidth}
\centering\footnotesize
\setlength{\tabcolsep}{3pt}
\renewcommand{\arraystretch}{1.40}
\resizebox{\linewidth}{!}{%
\begin{tabular}{@{}lcc@{}}
        \toprule
        Random mixing configuration & FID $\downarrow$ & CLIP $\uparrow$ \\
        \midrule
        Projected GAN & 14.80 & 0.328 \\
        $1{\times}1{+}3{\times}3$ & 14.48 & 0.328 \\
        $3{\times}3{+}5{\times}5$ & 13.35 & 0.333 \\
        \textbf{$1{\times}1{+}3{\times}3{+}5{\times}5$} & \scopechange{\textbf{11.03}} & 0.338 \\
        \bottomrule
    \end{tabular}}
\end{minipage} &
\begin{minipage}[c]{0.48\linewidth}
\centering
\includegraphics[width=0.98\linewidth]{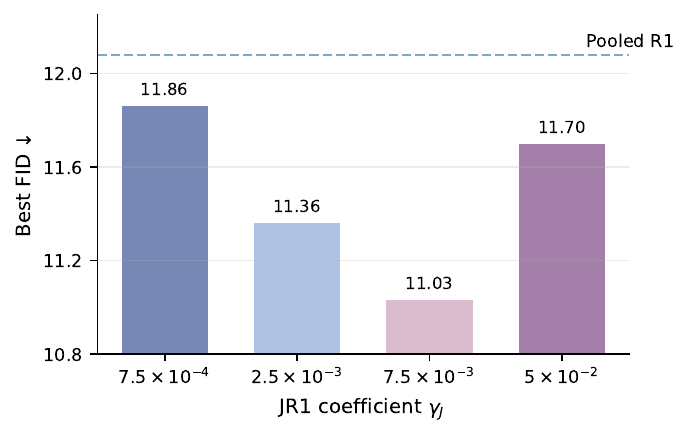}
\end{minipage} \\
\begin{minipage}[t]{0.48\linewidth}
\captionsetup[table]{position=bottom}
\captionsetup{hypcap=false}
\captionof{table}{\textbf{Multi-branch mixing design.}
SDXL ablations differing only in the composition of the random mixings $M$.}
\label{tab:projection}
\end{minipage} &
\begin{minipage}[t]{0.48\linewidth}
\captionsetup[figure]{position=bottom}
\captionsetup{hypcap=false}
\captionof{figure}{\scopechange{\textbf{JR1 coefficient sensitivity on SDXL.}
Best FID for each $\gamma_{\mathrm J}$; the dashed line is the pooled-R1 baseline.}}
\label{fig:jr1-coeff-sensitivity}
\end{minipage}
\end{tabular}
\par\medskip

\noindent\textbf{Analysis of multi-branch random mixing.}
The tested configurations compare different receptive-field supports. Both pairs
average two supports and differ only in which ones, yet $3{\times}3{+}5{\times}5$
reaches $13.35$ against $14.48$ for $1{\times}1{+}3{\times}3$, suggesting a benefit
from broader spatial mixing in this comparison. The triple then improves on
both pairs by at least \scopechange{$2.32$} FID, a margin larger than the $1.13$
separating the pairs. These results suggest complementary benefits from the
three kernels in the tested configurations: averaging kernels with different
supports combines random maps operating over different spatial extents.

\noindent\textbf{\scopechange{Sensitivity to the JR1 coefficient.}}
\scopechange{We scan the JR1 coefficient ($\gamma_{\mathrm J}$) on SDXL and
compare the best FID of each setting with the pooled-R1 baseline. Every
tested JR1 coefficient improves over pooled R1. The best result is obtained
at $\gamma_{\mathrm J}=7.5\times10^{-3}$, while the overall variation across
the scanned range remains modest, suggesting limited sensitivity within the
tested coefficient range under this SDXL configuration.}

\noindent\textbf{Discriminator augmentation under fixed mixing.}
\label{app:augmentation} 
Late-stage degradation can resemble discriminator overfitting, for which
data augmentation is a standard remedy. So we examine
DiffAugment~\citep{zhao2020diffaugment} and adaptive discriminator
augmentation (ADA)~\citep{karras2020ada} to test whether these remedies
alleviate the observed degradation.

We follow the SDXL protocol of Sec.~\ref{app:hparams} and compare
fixed mixing with no augmentation, DiffAugment, or ADA, alongside the
$5$k-refresh reference. All configurations share the Transformer critic
with pooled R1, training data, optimization settings, and a $40$k-step
budget. Performance is evaluated with the shared COCO-10K protocol
(Sec.~\ref{app:protocol}).

\begin{figure}[!htbp]
\centering
% Center content vertically in the first row; align captions in the second.
\begin{tabular}{@{}c@{\hspace{0.03\linewidth}}c@{}}
\begin{minipage}[c]{0.51\linewidth}
\centering
\includegraphics[width=\linewidth]{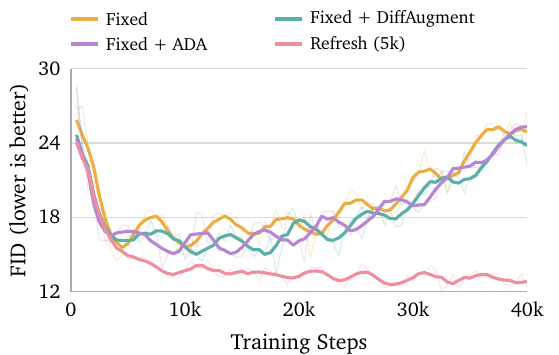}
\end{minipage} &
\begin{minipage}[c]{0.46\linewidth}
\centering\small
\setlength{\tabcolsep}{4pt}
\renewcommand{\arraystretch}{1.55}
\begin{tabular*}{\linewidth}{@{\extracolsep{\fill}}lr@{}}
    \toprule
    Mixing / augmentation & Best FID $\downarrow$ \\
    \midrule
    Fixed & 14.54 \\
    Fixed + DiffAugment & 14.46 \\
    Fixed + ADA & 14.38 \\
    Refresh ($5$k) & 12.08 \\
    \bottomrule
  \end{tabular*}
\end{minipage} \\
\begin{minipage}[t]{0.51\linewidth}
\captionsetup{position=bottom}
\caption{\textbf{Augmentation under fixed mixing.}
  Representative discriminator anti-overfitting interventions do not materially
  alleviate late-stage FID degradation.}
\label{fig:augmentation-controls}
\end{minipage} &
\begin{minipage}[t]{0.46\linewidth}
\captionsetup[table]{position=bottom}
\captionof{table}{\textbf{Best FID on COCO-10K.}
  Augmentation controls under fixed mixing, with $5$k refresh for
  reference. Unaugmented results are from Tab.~\ref{tab:ablation}(a).}
\label{tab:augmentation-controls}
\end{minipage}
\end{tabular}
\end{figure}

The FID trajectories (Fig.~\ref{fig:augmentation-controls}) show that
neither augmentation materially alleviates late-stage degradation under
fixed mixing; Tab.~\ref{tab:augmentation-controls} reports the best FID.
The persistence of this failure mode despite these standard remedies
supports the diagnosis in Sec.~\ref{sec:refresh}.
\FloatBarrier

\section{Conclusion and limitations}
\label{sec:conclusion}

\looseness=-1
{\color{hlblue}We modernize Projected GAN for one-step diffusion distillation with a
Lipschitz-controlled Transformer head, a periodic refresh of the random mixing,
and a Jacobian R1 penalty. The resulting method, ROAD, converges fast and attains the
best FID on both SDXL and SD3.5-M with a single adversarial objective and one
training stage, whereas competing methods combine several objectives and
stages. To our knowledge, ROAD is among the first to obtain a competitive
one-step SD3.5-M generator. Its preference scores, however, still trail those
of few-step competitors (Tab.~\ref{tab:sd35}). Text reaches the critic only
through cross-attention, so an explicit text--image alignment term is a natural
next step, though it would mean replacing DINOv3 with a contrastively trained
encoder such as
SigLIP2 (Tab.~\ref{tab:ablation}(b)). Scaling the frozen encoder does not help
by itself either, since within DINOv3 ViT-B beats ViT-S and ViT-L
(Tab.~\ref{tab:ablation}(b)). We attribute this to the two-player game's
balance, which our refresh mechanism improves but does not settle, and
addressing it directly is future work.}

\subsection*{AI use statement}
\phantomsection
\addcontentsline{toc}{subsection}{AI use statement}

In this work, we used generative AI tools to clean and reformat datasets, and provide
feedback on experiments.
We did not use generative AI tools to generate synthetic datasets; help
develop theoretical models or conceptual frameworks; formulate mathematical
claims; provide critical ingredients for proving mathematical claims; assist
in the writing of proofs; propose or refine hypotheses; implement methods; or
interpret results. Assistance with translation and support for qualitative
and thematic data analysis were not applicable to this work.

Additionally, we used generative AI tools to modify scientific figures,
create or edit software code, summarize or analyze existing literature,
edit the research paper to improve readability, and format references.

We carefully reviewed all AI-assisted work. AI-assisted code modifications
were manually inspected, debugged, and tested before being incorporated into
the final implementation. We also verified all edited portions of the
manuscript to ensure that our intended meaning was preserved. References
formatted with AI assistance were carefully checked for formatting accuracy
against the corresponding bibliographic information. We take responsibility
for the final content of this work, including text, claims, code, and other
artifacts produced with the aid of generative AI.

\subsection*{Reproducibility statement}
\phantomsection
\addcontentsline{toc}{subsection}{Reproducibility statement}
	\scopechange{Training configurations for both base models are listed in
	Sec.~\ref{app:hparams} as extracted from the launch scripts,
	including teacher, batch, learning rates, head and mixing configuration,
	and refresh interval; unavailable launch-script fields are marked in the
	table. The shared evaluation protocol for SDXL and SD3.5-M, including
	prompt set, resolution, FID reference, and CLIP encoder, is given in
	Sec.~\ref{app:protocol}, together with the checkpoint-selection rule
	(best FID on a 500-step evaluation grid, CLIP from the same checkpoint).}
	The JR1 estimator,
including probe distribution, count and lazy scaling, is
specified in Sec.~\ref{app:jr1}. The ROAD implementation and configuration files are publicly available at
\url{https://github.com/Lixiaoou729/ROAD}.

\bibliography{references}

\begin{thebibliography}{71}
\providecommand{\natexlab}[1]{#1}
\providecommand{\url}[1]{\texttt{#1}}
\expandafter\ifx\csname urlstyle\endcsname\relax
  \providecommand{\doi}[1]{doi: #1}\else
  \providecommand{\doi}{doi: \begingroup \urlstyle{rm}\Url}\fi

\bibitem[Yin et~al.(2024{\natexlab{a}})Yin, Gharbi, Park, Zhang, Shechtman,
  Durand, and Freeman]{yin2024dmd2}
Tianwei Yin, Micha{\"e}l Gharbi, Taesung Park, Richard Zhang, Eli Shechtman,
  Fr{\'e}do Durand, and William~T Freeman.
\newblock Improved distribution matching distillation for fast image synthesis.
\newblock In \emph{NeurIPS}, 2024{\natexlab{a}}.

\bibitem[Luo et~al.(2023)Luo, Hu, Zhang, Sun, Li, and
  Zhang]{luo2023diffinstruct}
Weijian Luo, Tianyang Hu, Shifeng Zhang, Jiacheng Sun, Zhenguo Li, and Zhihua
  Zhang.
\newblock {Diff-Instruct}: A universal approach for transferring knowledge from
  pre-trained diffusion models.
\newblock In \emph{NeurIPS}, 2023.

\bibitem[Yin et~al.(2024{\natexlab{b}})Yin, Gharbi, Zhang, Shechtman, Durand,
  Freeman, and Park]{yin2024dmd}
Tianwei Yin, Micha{\"e}l Gharbi, Richard Zhang, Eli Shechtman, Fr{\'e}do
  Durand, William~T Freeman, and Taesung Park.
\newblock One-step diffusion with distribution matching distillation.
\newblock In \emph{CVPR}, 2024{\natexlab{b}}.

\bibitem[Nguyen and Tran(2024)]{nguyen2024swiftbrush}
Thuan~Hoang Nguyen and Anh Tran.
\newblock {SwiftBrush}: One-step text-to-image diffusion model with variational
  score distillation.
\newblock In \emph{CVPR}, 2024.

\bibitem[Zhou et~al.(2024)Zhou, Zheng, Wang, Yin, and Huang]{zhou2024sid}
Mingyuan Zhou, Huangjie Zheng, Zhendong Wang, Mingzhang Yin, and Hai Huang.
\newblock Score identity distillation: Exponentially fast distillation of
  pretrained diffusion models for one-step generation.
\newblock In \emph{ICML}, 2024.

\bibitem[Sauer et~al.(2024{\natexlab{a}})Sauer, Lorenz, Blattmann, and
  Rombach]{sauer2024add}
Axel Sauer, Dominik Lorenz, Andreas Blattmann, and Robin Rombach.
\newblock Adversarial diffusion distillation.
\newblock In \emph{ECCV}, 2024{\natexlab{a}}.

\bibitem[Zhou et~al.(2025{\natexlab{a}})Zhou, Zheng, Gu, Wang, and
  Huang]{zhou2024sida}
Mingyuan Zhou, Huangjie Zheng, Yi~Gu, Zhendong Wang, and Hai Huang.
\newblock Adversarial score identity distillation: Rapidly surpassing the
  teacher in one step.
\newblock In \emph{ICLR}, 2025{\natexlab{a}}.

\bibitem[Sauer et~al.(2024{\natexlab{b}})Sauer, Boesel, Dockhorn, Blattmann,
  Esser, and Rombach]{sauer2024ladd}
Axel Sauer, Frederic Boesel, Tim Dockhorn, Andreas Blattmann, Patrick Esser,
  and Robin Rombach.
\newblock Fast high-resolution image synthesis with latent adversarial
  diffusion distillation.
\newblock In \emph{SIGGRAPH Asia 2024 Conference Papers}, 2024{\natexlab{b}}.

\bibitem[Chen et~al.(2025)Chen, Bandyopadhyay, Zou, and
  Song]{chen2025nitrofusion}
Dar-Yen Chen, Hmrishav Bandyopadhyay, Kai Zou, and Yi-Zhe Song.
\newblock {NitroFusion}: High-fidelity single-step diffusion through dynamic
  adversarial training.
\newblock In \emph{CVPR}, 2025.

\bibitem[Lin et~al.(2025)Lin, Yang, Li, Xia, et~al.]{lin2025apt}
Shanchuan Lin, Xin Yang, Yuxi Li, Xing Xia, et~al.
\newblock Diffusion adversarial post-training for one-step video generation.
\newblock In \emph{ICML}, 2025.

\bibitem[Sauer et~al.(2021)Sauer, Chitta, M{\"u}ller, and
  Geiger]{sauer2021projected}
Axel Sauer, Kashyap Chitta, Jonas M{\"u}ller, and Andreas Geiger.
\newblock Projected {GANs} converge faster.
\newblock In \emph{NeurIPS}, 2021.

\bibitem[Sauer et~al.(2023)Sauer, Karras, Laine, Geiger, and
  Aila]{sauer2023stylegant}
Axel Sauer, Tero Karras, Samuli Laine, Andreas Geiger, and Timo Aila.
\newblock {StyleGAN-T}: Unlocking the power of {GANs} for fast large-scale
  text-to-image synthesis.
\newblock In \emph{ICML}, 2023.

\bibitem[Sim{\'e}oni et~al.(2026)Sim{\'e}oni, Vo, Seitzer, Baldassarre, Oquab,
  Jose, Khalidov, Szafraniec, Yi, Ramamonjisoa, Massa, Haziza, Wehrstedt, Wang,
  Darcet, Moutakanni, Sentana, Roberts, Vedaldi, Tolan, Brandt, Couprie,
  Mairal, J{\'e}gou, Labatut, and Bojanowski]{simeoni2025dinov3}
Oriane Sim{\'e}oni, Huy~V. Vo, Maximilian Seitzer, Federico Baldassarre, Maxime
  Oquab, Cijo Jose, Vasil Khalidov, Marc Szafraniec, Seungeun Yi, Micha{\"e}l
  Ramamonjisoa, Francisco Massa, Daniel Haziza, Luca Wehrstedt, Jianyuan Wang,
  Timoth{\'e}e Darcet, Th{\'e}o Moutakanni, Leonel Sentana, Claire Roberts,
  Andrea Vedaldi, Jamie Tolan, John Brandt, Camille Couprie, Julien Mairal,
  Herv{\'e} J{\'e}gou, Patrick Labatut, and Piotr Bojanowski.
\newblock {DINOv3}.
\newblock \emph{TMLR}, 2026.

\bibitem[Bolya et~al.(2025)Bolya, Huang, Sun, Cho, Madotto, Wei, Ma, Zhi,
  Rajasegaran, Rasheed, Wang, Monteiro, Xu, Dong, Ravi, Li, Doll{\'a}r, and
  Feichtenhofer]{bolya2025pe}
Daniel Bolya, Po-Yao Huang, Peize Sun, Jang~Hyun Cho, Andrea Madotto, Chen Wei,
  Tengyu Ma, Jiale Zhi, Jathushan Rajasegaran, Hanoona Rasheed, Junke Wang,
  Marco Monteiro, Hu~Xu, Shiyu Dong, Nikhila Ravi, Daniel Li, Piotr Doll{\'a}r,
  and Christoph Feichtenhofer.
\newblock {Perception Encoder}: The best visual embeddings are not at the
  output of the network.
\newblock In \emph{NeurIPS}, 2025.

\bibitem[Mescheder et~al.(2018)Mescheder, Geiger, and Nowozin]{mescheder2018r1}
Lars Mescheder, Andreas Geiger, and Sebastian Nowozin.
\newblock Which training methods for {GANs} do actually converge?
\newblock In \emph{ICML}, 2018.

\bibitem[Karras et~al.(2020{\natexlab{a}})Karras, Laine, Aittala, Hellsten,
  Lehtinen, and Aila]{karras2020stylegan2}
Tero Karras, Samuli Laine, Miika Aittala, Janne Hellsten, Jaakko Lehtinen, and
  Timo Aila.
\newblock Analyzing and improving the image quality of {StyleGAN}.
\newblock In \emph{CVPR}, 2020{\natexlab{a}}.

\bibitem[Kang et~al.(2024)Kang, Zhang, Barnes, Paris, Kwak, Park, Shechtman,
  Zhu, and Park]{kang2024diffusion2gan}
Minguk Kang, Richard Zhang, Connelly Barnes, Sylvain Paris, Suha Kwak, Jaesik
  Park, Eli Shechtman, Jun-Yan Zhu, and Taesung Park.
\newblock Distilling diffusion models into conditional {GANs}.
\newblock In \emph{ECCV}, 2024.

\bibitem[Zheng and Yang(2025)]{zheng2025d2o}
Bowen Zheng and Tianming Yang.
\newblock Revisiting diffusion models: From generative pre-training to one-step
  generation.
\newblock In \emph{ICML}, 2025.

\bibitem[Suvorov et~al.(2022)Suvorov, Logacheva, Mashikhin, Remizova, Ashukha,
  Silvestrov, Kong, Goka, Park, and Lempitsky]{suvorov2022lama}
Roman Suvorov, Elizaveta Logacheva, Anton Mashikhin, Anastasia Remizova,
  Arsenii Ashukha, Aleksei Silvestrov, Naejin Kong, Harshith Goka, Kiwoong
  Park, and Victor Lempitsky.
\newblock Resolution-robust large mask inpainting with fourier convolutions.
\newblock In \emph{WACV}, 2022.

\bibitem[Kang et~al.(2023)Kang, Zhu, Zhang, Park, Shechtman, Paris, and
  Park]{kang2023gigagan}
Minguk Kang, Jun-Yan Zhu, Richard Zhang, Jaesik Park, Eli Shechtman, Sylvain
  Paris, and Taesung Park.
\newblock Scaling up {GANs} for text-to-image synthesis.
\newblock In \emph{CVPR}, 2023.

\bibitem[Goodfellow et~al.(2014)Goodfellow, Pouget-Abadie, Mirza, Xu,
  Warde-Farley, Ozair, Courville, and Bengio]{goodfellow2014gan}
Ian Goodfellow, Jean Pouget-Abadie, Mehdi Mirza, Bing Xu, David Warde-Farley,
  Sherjil Ozair, Aaron Courville, and Yoshua Bengio.
\newblock Generative adversarial nets.
\newblock In \emph{NeurIPS}, 2014.

\bibitem[Zhou et~al.(2019)Zhou, Liang, Song, Yu, Wang, Zhang, Yu, and
  Zhang]{zhou2019lipschitz}
Zhiming Zhou, Jiadong Liang, Yuxuan Song, Lantao Yu, Hongwei Wang, Weinan
  Zhang, Yong Yu, and Zhihua Zhang.
\newblock Lipschitz generative adversarial nets.
\newblock In \emph{ICML}, volume~97 of \emph{Proceedings of Machine Learning
  Research}, pages 7584--7593, 2019.

\bibitem[Miyato et~al.(2018)Miyato, Kataoka, Koyama, and Yoshida]{miyato2018sn}
Takeru Miyato, Toshiki Kataoka, Masanori Koyama, and Yuichi Yoshida.
\newblock Spectral normalization for generative adversarial networks.
\newblock In \emph{ICLR}, 2018.

\bibitem[Kim et~al.(2021)Kim, Papamakarios, and Mnih]{kim2021lipschitz}
Hyunjik Kim, George Papamakarios, and Andriy Mnih.
\newblock The {Lipschitz} constant of self-attention.
\newblock In \emph{ICML}, 2021.

\bibitem[Liu et~al.(2022)Liu, Hu, Lin, Yao, Xie, et~al.]{liu2022swinv2}
Ze~Liu, Han Hu, Yutong Lin, Zhuliang Yao, Zhenda Xie, et~al.
\newblock Swin transformer {V2}: Scaling up capacity and resolution.
\newblock In \emph{CVPR}, 2022.

\bibitem[Lee et~al.(2022)Lee, Chang, Jiang, Zhang, Tu, and Liu]{lee2022vitgan}
Kwonjoon Lee, Huiwen Chang, Lu~Jiang, Han Zhang, Zhuowen Tu, and Ce~Liu.
\newblock {ViTGAN}: Training {GANs} with vision transformers.
\newblock In \emph{ICLR}, 2022.

\bibitem[Zhao et~al.(2020)Zhao, Liu, Lin, Zhu, and Han]{zhao2020diffaugment}
Shengyu Zhao, Zhijian Liu, Ji~Lin, Jun-Yan Zhu, and Song Han.
\newblock Differentiable augmentation for data-efficient {GAN} training.
\newblock In \emph{NeurIPS}, 2020.

\bibitem[Karras et~al.(2020{\natexlab{b}})Karras, Aittala, Hellsten, Laine,
  Lehtinen, and Aila]{karras2020ada}
Tero Karras, Miika Aittala, Janne Hellsten, Samuli Laine, Jaakko Lehtinen, and
  Timo Aila.
\newblock Training generative adversarial networks with limited data.
\newblock In \emph{NeurIPS}, 2020{\natexlab{b}}.

\bibitem[Novak et~al.(2018)Novak, Bahri, Abolafia, Pennington, and
  Sohl-Dickstein]{novak2018sensitivity}
Roman Novak, Yasaman Bahri, Daniel~A. Abolafia, Jeffrey Pennington, and Jascha
  Sohl-Dickstein.
\newblock Sensitivity and generalization in neural networks: an empirical
  study.
\newblock In \emph{ICLR}, 2018.

\bibitem[Varga et~al.(2018)Varga, Csisz{\'a}rik, and
  Zombori]{varga2018gradient}
D{\'a}niel Varga, Adri{\'a}n Csisz{\'a}rik, and Zsolt Zombori.
\newblock Gradient regularization improves accuracy of discriminative models.
\newblock \emph{Schedae Informaticae}, 27:\penalty0 31--45, 2018.

\bibitem[Hutchinson(1989)]{hutchinson1989}
Michael~F Hutchinson.
\newblock A stochastic estimator of the trace of the influence matrix for
  {Laplacian} smoothing splines.
\newblock \emph{Communications in Statistics---Simulation and Computation},
  18\penalty0 (3):\penalty0 1059--1076, 1989.

\bibitem[Hoffman et~al.(2019)Hoffman, Roberts, and Yaida]{hoffman2019jacobian}
Judy Hoffman, Daniel~A Roberts, and Sho Yaida.
\newblock Robust learning with {Jacobian} regularization.
\newblock \emph{arXiv:1908.02729}, 2019.

\bibitem[Podell et~al.(2024)Podell, English, Lacey, Blattmann, Dockhorn,
  M{\"u}ller, Penna, and Rombach]{podell2024sdxl}
Dustin Podell, Zion English, Kyle Lacey, Andreas Blattmann, Tim Dockhorn, Jonas
  M{\"u}ller, Joe Penna, and Robin Rombach.
\newblock {SDXL}: Improving latent diffusion models for high-resolution image
  synthesis.
\newblock In \emph{ICLR}, 2024.

\bibitem[Schuhmann et~al.(2022)]{schuhmann2022laion5b}
Christoph Schuhmann et~al.
\newblock {LAION-5B}: An open large-scale dataset for training next generation
  image-text models.
\newblock In \emph{NeurIPS}, 2022.

\bibitem[Zhou et~al.(2025{\natexlab{b}})Zhou, Wang, Zheng, and
  Huang]{zhou2024sidlsg}
Mingyuan Zhou, Zhendong Wang, Huangjie Zheng, and Hai Huang.
\newblock Long and short guidance in score identity distillation for one-step
  text-to-image generation.
\newblock In \emph{ICLR}, 2025{\natexlab{b}}.

\bibitem[Kirstain et~al.(2023)Kirstain, Polyak, Singer, Matiana, Penna, and
  Levy]{kirstain2023pickscore}
Yuval Kirstain, Adam Polyak, Uriel Singer, Shahbuland Matiana, Joe Penna, and
  Omer Levy.
\newblock {Pick-a-Pic}: An open dataset of user preferences for text-to-image
  generation.
\newblock In \emph{NeurIPS}, 2023.

\bibitem[Wu et~al.(2023)Wu, Hao, Sun, Chen, Zhu, Zhao, and Li]{wu2023hpsv2}
Xiaoshi Wu, Yiming Hao, Keqiang Sun, Yixiong Chen, Feng Zhu, Rui Zhao, and
  Hongsheng Li.
\newblock Human preference score v2: A solid benchmark for evaluating human
  preferences of text-to-image synthesis.
\newblock \emph{arXiv preprint arXiv:2306.09341}, 2023.

\bibitem[Lu et~al.(2025)Lu, Ren, Lin, Xia, Wang, Xiao, Ma, Xie, and
  Lai]{lu2025adm}
Yanzuo Lu, Yuxi Ren, Shanchuan Lin, Xin Xia, Xing Wang, Xuefeng Xiao, Andy~J
  Ma, Xiaohua Xie, and Jian-Huang Lai.
\newblock Adversarial distribution matching for diffusion distillation towards
  efficient image and video synthesis.
\newblock In \emph{ICCV}, 2025.

\bibitem[Ren et~al.(2024)Ren, Xia, Lu, Zhang, Wu, Xie, Wang, and
  Xiao]{ren2024hypersd}
Yuxi Ren, Xin Xia, Yanzuo Lu, Jiacheng Zhang, Jie Wu, Pan Xie, Xing Wang, and
  Xuefeng Xiao.
\newblock {Hyper-SD}: Trajectory segmented consistency model for efficient
  image synthesis.
\newblock In \emph{NeurIPS}, 2024.

\bibitem[Xu et~al.(2023)Xu, Liu, Wu, Tong, Li, Ding, Tang, and
  Dong]{xu2023imagereward}
Jiazheng Xu, Xiao Liu, Yuchen Wu, Yuxuan Tong, Qinkai Li, Ming Ding, Jie Tang,
  and Yuxiao Dong.
\newblock {ImageReward}: Learning and evaluating human preferences for
  text-to-image generation.
\newblock In \emph{NeurIPS}, 2023.

\bibitem[Ghosh et~al.(2023)Ghosh, Hajishirzi, and Schmidt]{ghosh2023geneval}
Dhruba Ghosh, Hannaneh Hajishirzi, and Ludwig Schmidt.
\newblock {GenEval}: An object-focused framework for evaluating text-to-image
  alignment.
\newblock In \emph{NeurIPS}, 2023.

\bibitem[Starodubcev et~al.(2026)Starodubcev, Drobyshevskiy, Kuznedelev,
  Babenko, and Baranchuk]{swd2025}
Nikita Starodubcev, Ilya Drobyshevskiy, Denis Kuznedelev, Artem Babenko, and
  Dmitry Baranchuk.
\newblock Scale-wise distillation of diffusion models.
\newblock In \emph{ICLR}, 2026.

\bibitem[Lin et~al.(2024)Lin, Wang, and Yang]{lin2024sdxllightning}
Shanchuan Lin, Anran Wang, and Xiao Yang.
\newblock {SDXL-Lightning}: Progressive adversarial diffusion distillation.
\newblock \emph{arXiv:2402.13929}, 2024.

\bibitem[Ge et~al.(2026)]{ge2026senseflow}
Xingtong Ge et~al.
\newblock {SenseFlow}: Scaling distribution matching for flow-based
  text-to-image distillation.
\newblock In \emph{ICLR}, 2026.

\bibitem[Zhou et~al.(2025{\natexlab{c}})Zhou, Gu, Zheng, Song, He, Zhang, Hu,
  and Yang]{zhou2025scoreflow}
Mingyuan Zhou, Yi~Gu, Huangjie Zheng, Liangchen Song, Guande He, Yizhe Zhang,
  Wenze Hu, and Yinfei Yang.
\newblock Score distillation of flow matching models.
\newblock \emph{arXiv:2509.25127}, 2025{\natexlab{c}}.

\bibitem[Tschannen et~al.(2025)Tschannen, Gritsenko, Wang, Naeem,
  Alabdulmohsin, Parthasarathy, Evans, Beyer, Xia, Mustafa, H{\'e}naff,
  Harmsen, Steiner, and Zhai]{tschannen2025siglip2}
Michael Tschannen, Alexey Gritsenko, Xiao Wang, Muhammad~Ferjad Naeem, Ibrahim
  Alabdulmohsin, Nikhil Parthasarathy, Talfan Evans, Lucas Beyer, Ye~Xia, Basil
  Mustafa, Olivier H{\'e}naff, Jeremiah Harmsen, Andreas Steiner, and Xiaohua
  Zhai.
\newblock {SigLIP 2}: Multilingual vision-language encoders with improved
  semantic understanding, localization, and dense features.
\newblock \emph{arXiv:2502.14786}, 2025.

\bibitem[Chong et~al.(2024)Chong, Singh, Li, Lu, and Forsyth]{chong2024p2d}
Min~Jin Chong, Krishna~Kumar Singh, Yijun Li, Jingwan Lu, and David Forsyth.
\newblock {P2D}: Plug and play discriminator for accelerating {GAN} frameworks.
\newblock In \emph{WACV}, 2024.

\bibitem[He et~al.(2015)He, Zhang, Ren, and Sun]{he2015delving}
Kaiming He, Xiangyu Zhang, Shaoqing Ren, and Jian Sun.
\newblock Delving deep into rectifiers: Surpassing human-level performance on
  {ImageNet} classification.
\newblock In \emph{ICCV}, pages 1026--1034, 2015.

\bibitem[Araujo et~al.(2021)Araujo, Negrevergne, Chevaleyre, and
  Atif]{araujo2021lipschitz}
Alexandre Araujo, Benjamin Negrevergne, Yann Chevaleyre, and Jamal Atif.
\newblock On {Lipschitz} regularization of convolutional layers using
  {Toeplitz} matrix theory.
\newblock In \emph{AAAI}, volume~35, pages 6661--6669, 2021.

\bibitem[Sedghi et~al.(2019)Sedghi, Gupta, and Long]{sedghi2019singular}
Hanie Sedghi, Vineet Gupta, and Philip~M. Long.
\newblock The singular values of convolutional layers.
\newblock In \emph{ICLR}, 2019.

\bibitem[Ding et~al.(2021{\natexlab{a}})Ding, Zhang, Han, and
  Ding]{ding2021dbb}
Xiaohan Ding, Xiangyu Zhang, Jungong Han, and Guiguang Ding.
\newblock Diverse branch block: Building a convolution as an inception-like
  unit.
\newblock In \emph{CVPR}, pages 10886--10895, 2021{\natexlab{a}}.

\bibitem[Ding et~al.(2021{\natexlab{b}})Ding, Zhang, Ma, Han, Ding, and
  Sun]{ding2021repvgg}
Xiaohan Ding, Xiangyu Zhang, Ningning Ma, Jungong Han, Guiguang Ding, and Jian
  Sun.
\newblock {RepVGG}: Making {VGG}-style convnets great again.
\newblock In \emph{CVPR}, pages 13733--13742, 2021{\natexlab{b}}.

\bibitem[Marchenko and Pastur(1967)]{marchenko1967}
Vladimir~A Marchenko and Leonid~A Pastur.
\newblock Distribution of eigenvalues for some sets of random matrices.
\newblock \emph{Mathematics of the USSR-Sbornik}, 1\penalty0 (4):\penalty0
  457--483, 1967.

\bibitem[Simonyan et~al.(2014)Simonyan, Vedaldi, and
  Zisserman]{simonyan2014saliency}
Karen Simonyan, Andrea Vedaldi, and Andrew Zisserman.
\newblock Deep inside convolutional networks: Visualising image classification
  models and saliency maps.
\newblock In \emph{ICLR Workshop}, 2014.

\bibitem[Sauer et~al.(2022)Sauer, Schwarz, and Geiger]{sauer2022styleganxl}
Axel Sauer, Katja Schwarz, and Andreas Geiger.
\newblock {StyleGAN-XL}: Scaling {StyleGAN} to large diverse datasets.
\newblock In \emph{SIGGRAPH}, 2022.

\bibitem[Jiang et~al.(2021)Jiang, Chang, and Wang]{jiang2021transgan}
Yifan Jiang, Shiyu Chang, and Zhangyang Wang.
\newblock {TransGAN}: Two pure transformers can make one strong {GAN}, and that
  can scale up.
\newblock In \emph{NeurIPS}, 2021.

\bibitem[Zhang et~al.(2022)Zhang, Gu, Zhang, Bao, Chen, Wen, Wang, and
  Guo]{zhang2022styleswin}
Bowen Zhang, Shuyang Gu, Bo~Zhang, Jianmin Bao, Dong Chen, Fang Wen, Yong Wang,
  and Baining Guo.
\newblock {StyleSwin}: Transformer-based {GAN} for high-resolution image
  generation.
\newblock In \emph{CVPR}, 2022.

\bibitem[Zhao et~al.(2021)Zhao, Zhang, Chen, Metaxas, and Zhang]{zhao2021hit}
Long Zhao, Zhixing Zhang, Ting Chen, Dimitris~N Metaxas, and Han Zhang.
\newblock Improved transformer for high-resolution {GANs}.
\newblock In \emph{NeurIPS}, 2021.

\bibitem[Park and Kim(2022)]{park2022styleformer}
Jeeseung Park and Younggeun Kim.
\newblock {Styleformer}: Transformer based generative adversarial networks with
  style vector.
\newblock In \emph{CVPR}, 2022.

\bibitem[Hudson and Zitnick(2021)]{hudson2021ganformer}
Drew~A Hudson and C~Lawrence Zitnick.
\newblock Generative adversarial transformers.
\newblock In \emph{ICML}, volume 139 of \emph{Proceedings of Machine Learning
  Research}, pages 4487--4499, 2021.

\bibitem[Qi et~al.(2023)Qi, Wang, Chen, Shi, and Zhang]{qi2023lipsformer}
Xianbiao Qi, Jianan Wang, Yihao Chen, Yukai Shi, and Lei Zhang.
\newblock {LipsFormer}: Introducing {Lipschitz} continuity to vision
  transformers.
\newblock In \emph{ICLR}, 2023.

\bibitem[Neyshabur et~al.(2017)Neyshabur, Bhojanapalli, and
  Chakrabarti]{neyshabur2017mrp}
Behnam Neyshabur, Srinadh Bhojanapalli, and Ayan Chakrabarti.
\newblock Stabilizing {GAN} training with multiple random projections.
\newblock \emph{arXiv:1705.07831}, 2017.

\bibitem[Cannings and Samworth(2017)]{cannings2017random}
Timothy~I Cannings and Richard~J Samworth.
\newblock Random-projection ensemble classification.
\newblock \emph{Journal of the Royal Statistical Society Series B}, 2017.

\bibitem[Bandyopadhyay et~al.(2025)]{bandyopadhyay2025sd35flash}
Hmrishav Bandyopadhyay et~al.
\newblock {SD3.5-Flash}: Distribution-guided distillation of generative flows.
\newblock \emph{arXiv:2509.21318}, 2025.

\bibitem[Nikishin et~al.(2022)Nikishin, Schwarzer, D'Oro, Bacon, and
  Courville]{nikishin2022primacy}
Evgenii Nikishin, Max Schwarzer, Pierluca D'Oro, Pierre-Luc Bacon, and Aaron
  Courville.
\newblock The primacy bias in deep reinforcement learning.
\newblock In \emph{ICML}, 2022.

\bibitem[Sokar et~al.(2023)Sokar, Agarwal, Castro, and Evci]{sokar2023dormant}
Ghada Sokar, Rishabh Agarwal, Pablo~Samuel Castro, and Utku Evci.
\newblock The dormant neuron phenomenon in deep reinforcement learning.
\newblock In \emph{ICML}, 2023.

\bibitem[Ash and Adams(2020)]{ash2020warm}
Jordan~T Ash and Ryan~P Adams.
\newblock On warm-starting neural network training.
\newblock In \emph{NeurIPS}, 2020.

\bibitem[Dohare et~al.(2024)Dohare, Hernandez-Garcia, Lan, Rahman, Mahmood, and
  Sutton]{dohare2024plasticity}
Shibhansh Dohare, J~Fernando Hernandez-Garcia, Qingfeng Lan, Parash Rahman,
  A~Rupam Mahmood, and Richard~S Sutton.
\newblock Loss of plasticity in deep continual learning.
\newblock \emph{Nature}, 632:\penalty0 768--774, 2024.

\bibitem[Roth et~al.(2017)Roth, Lucchi, Nowozin, and
  Hofmann]{roth2017stabilizing}
Kevin Roth, Aurelien Lucchi, Sebastian Nowozin, and Thomas Hofmann.
\newblock Stabilizing training of generative adversarial networks through
  regularization.
\newblock In \emph{NeurIPS}, 2017.

\bibitem[Huang et~al.(2024)Huang, Gokaslan, Kuleshov, and
  Tompkin]{huang2024r3gan}
Yiwen Huang, Aaron Gokaslan, Volodymyr Kuleshov, and James Tompkin.
\newblock The {GAN} is dead; long live the {GAN}! {A} modern {GAN} baseline.
\newblock In \emph{NeurIPS}, 2024.

\bibitem[Jakubovitz and Giryes(2018)]{jakubovitz2018jacobian}
Daniel Jakubovitz and Raja Giryes.
\newblock Improving {DNN} robustness to adversarial attacks using {Jacobian}
  regularization.
\newblock In \emph{ECCV}, pages 514--529, 2018.

\end{thebibliography}
\bibliographystyle{unsrtnat}

\FloatBarrier
\appendix
\section{Lipschitz Analysis of the Transformer Head}
\label{app:lipschitz}

{\color{revtext}
This appendix supports Sec.~\ref{sec:head}. We bound the sensitivity of one
Transformer head, viewed as the map from the mixed tokens of its feature level
to the dense logits of that level at a fixed prompt.

\noindent\textbf{Setting.}
Fix a feature level $\ell$ and a prompt $c$. The head reads the mixed tokens
$X=\hat f_\ell\in\mathbb R^{N_\ell\times d}$ of
Eq.~(\plaineqref{eq:mixing}), where $d$ is the token width, and returns the
$N_\ell$ logits $D^\ell(X,c)\in\mathbb R^{N_\ell}$ of level $\ell$. The head
has $B_\ell$ pre-norm blocks (Sec.~\ref{app:road-head}). Each block
contains a cosine self-attention branch followed by an FFN branch. The first
block also contains a cosine cross-attention branch between the two. Its
queries are the patch tokens, and its keys and values are computed from the
$N_c$ text tokens of $c$. Every branch applies RMSNorm to its input, passes the
result through its sublayer, multiplies the output channel-wise by a learnable
scale in $\mathbb R^d$, and adds the product to the input. We write
$\alpha^{(b)}_{\mathrm{sa}}$ and $\alpha^{(b)}_{\mathrm{ffn}}$ for the residual
scales of block $b$ and $\alpha_{\mathrm{ca}}$ for the cross-attention gate, and
we set $a^{(b)}_{\mathrm{sa}}=\lVert\alpha^{(b)}_{\mathrm{sa}}\rVert_\infty$,
$a^{(b)}_{\mathrm{ffn}}=\lVert\alpha^{(b)}_{\mathrm{ffn}}\rVert_\infty$, and
$a_{\mathrm{ca}}=\lVert\alpha_{\mathrm{ca}}\rVert_\infty$. A final RMSNorm and
the linear logit head of Sec.~\ref{app:road-head} produce the logits.

We measure token matrices in the Frobenius norm, logit vectors in the
Euclidean norm, and Jacobians in the operator norm induced by these two norms.
We write $\mathrm{Lip}(f;\mathcal S)$ for the Lipschitz constant of $f$ on a
set $\mathcal S$. The RMSNorm layers use $\varepsilon>0$ and learnable weights $w$
with $\lVert w\rVert_\infty\le\bar w$, and we set
$c_{\mathrm n}=\bar w/\sqrt{\varepsilon}$ and $\nu=\sqrt d\,\bar w$. The FFN
activation is GELU, and
$c_{\mathrm{ffn}}=\sup_{t\in\mathbb R}\lvert\mathrm{GELU}'(t)\rvert\le1.13$.
Attention head $h$ has temperature $s_h$, and its clamped parametrization
(Sec.~\ref{app:road-head}) guarantees $0<s_h\le s_{\max}$. In cross-attention head $h$, let $V_{c,h}$ be
the value matrix computed from the text tokens of $c$, and let
$\kappa_c=\max_h\lVert V_{c,h}\rVert_2$.

\begin{lemma}[Conditional Lipschitz bound for a Transformer head]
\label{lem:lip}
Assume that every spectrally normalized linear map $W$ of the head satisfies
$\lVert W\rVert_2\le1$. Let $\mathcal S\subset\mathbb R^{N_\ell\times d}$ be a
convex set such that, for every $X\in\mathcal S$, every query and key vector
has norm at least $\rho>0$ immediately before $\ell_2$-normalization. Set
\[
c_{\mathrm{sa}}=\sqrt{N_\ell}\,\Bigl(1+\frac{s_{\max}\sqrt{N_\ell}\,\nu}{\rho}\Bigr),
\qquad
c_{\mathrm{ca}}=\frac{s_{\max}\sqrt{N_c}\,\kappa_c}{2\rho}.
\]
Then the Jacobian of $D^\ell(\cdot,c)$ has norm at most $\Lambda_\ell$ at every
$X\in\mathcal S$, and $\mathrm{Lip}\bigl(D^\ell(\cdot,c);\mathcal S\bigr)\le\Lambda_\ell$,
where
\begin{equation}
\Lambda_\ell
=c_{\mathrm n}\bigl(1+a_{\mathrm{ca}}c_{\mathrm n}c_{\mathrm{ca}}\bigr)
\prod_{b=1}^{B_\ell}\bigl(1+a_{\mathrm{sa}}^{(b)}c_{\mathrm n}c_{\mathrm{sa}}\bigr)
\bigl(1+a_{\mathrm{ffn}}^{(b)}c_{\mathrm n}c_{\mathrm{ffn}}\bigr).
\label{eq:lip}
\end{equation}
\end{lemma}

The cross-attention factor is absent in a head without cross-attention. At
initialization, the gate satisfies $a_{\mathrm{ca}}=0$, and every residual
scale equals $1/B_\ell$ (Sec.~\ref{app:road-head}). Since
$(1+t/B_\ell)^{B_\ell}\le e^t$ for $t\ge0$, the bound at initialization
satisfies
$\Lambda_\ell\le c_{\mathrm n}\exp\bigl(c_{\mathrm n}(c_{\mathrm{sa}}+c_{\mathrm{ffn}})\bigr)$
for every depth $B_\ell$. The residual scaling therefore keeps the initial
bound below a constant that does not depend on the depth of the head.

\begin{proof}
All bounds below concern Jacobians at a point of $\mathcal S$, and
$\delta(\cdot)$ denotes the first-order change caused by a perturbation
$\delta X$. By the chain rule, the Jacobian norm of the head is at most the
product of the Jacobian norms of its layers. Because $\mathcal S$ is convex,
integrating this bound along segments in $\mathcal S$ gives the Lipschitz
bound.

\emph{Step 1: RMSNorm and FFN.}
For one token $x\in\mathbb R^d$, let
$r(x)=\sqrt{\lVert x\rVert_2^2/d+\varepsilon}$, so that
$\mathrm{RMSNorm}(x)=w\odot x/r(x)$. Its Jacobian is
\[
\frac{\mathrm{diag}(w)}{r(x)}\Bigl(I_d-\frac{xx^\top}{d\,r(x)^2}\Bigr).
\]
The matrix in parentheses has eigenvalue $1$ on $x^\perp$ and eigenvalue
$\varepsilon/r(x)^2\le1$ along $x$. Since $r(x)\ge\sqrt\varepsilon$, the
Jacobian has norm at most $c_{\mathrm n}$. Since
$\lVert x\rVert_2<\sqrt d\,r(x)$, every output token has norm at most $\nu$.
RMSNorm acts on each token separately. On token matrices, its Jacobian
 also has norm at most $c_{\mathrm n}$, and
$\lVert\mathrm{RMSNorm}(X)\rVert_F\le\sqrt{N_\ell}\,\nu$. The FFN composes two
linear maps of norm at most one with GELU, so its Jacobian has norm at most
$c_{\mathrm{ffn}}$.

\emph{Step 2: Cosine self-attention.}
Two primitives are shared with cross-attention. The map $u\mapsto u/\lVert
u\rVert_2$ has Jacobian $(I-\widehat u\,\widehat u^\top)/\lVert u\rVert_2$
with $\widehat u=u/\lVert u\rVert_2$, whose norm is at most $1/\rho$ when
$\lVert u\rVert_2\ge\rho$. At a probability vector $p$, softmax has Jacobian
$\mathrm{diag}(p)-pp^\top$, whose largest eigenvalue is at most $1/2$. Row-wise
softmax therefore has Jacobian norm at most $1/2$.

Let $Z=\mathrm{RMSNorm}(X)$ be the input of the sublayer. For head
$h=1,\dots,H$, let
\[
Q_h=ZW_{Q,h},\qquad K_h=ZW_{K,h},\qquad V_h=ZW_{V,h},\qquad
P_h=\mathrm{softmax}(A_h),
\]
where $A_h=s_h\widehat Q_h\widehat K_h^\top$, the matrices $\widehat Q_h$ and
$\widehat K_h$ normalize each row, with rows $\widehat q_{h,i}$ and
$\widehat k_{h,j}$, and softmax acts on each row. The sublayer output is $[P_1V_1,\dots,P_HV_H]W_O$. Rotary position
embeddings rotate each query and key by an orthogonal matrix, so they change
none of the norms below. Each per-head projection is a column block of its full
projection, so that, for example,
$\sum_h\lVert\delta V_h\rVert_F^2=\lVert\delta Z\,W_V\rVert_F^2\le\lVert\delta
Z\rVert_F^2$. The perturbation of each head output splits into a value path
and a weight path, $\delta(P_hV_h)=P_h\,\delta V_h+\delta P_h\,V_h$.

For the value path, each row of $P_h$ is a probability vector, so
$\lVert P_h\rVert_2\le\lVert P_h\rVert_F\le\sqrt{N_\ell}$. The value path
therefore contributes at most $\sqrt{N_\ell}\,\lVert\delta Z\rVert_F$.
For the weight path, the normalized rows have unit norm, so
$\lvert\delta A_{h,ij}\rvert\le s_{\max}\bigl(\lVert\delta\widehat
q_{h,i}\rVert_2+\lVert\delta\widehat k_{h,j}\rVert_2\bigr)$. Squaring with
$(a+b)^2\le2a^2+2b^2$, summing over $i$, $j$ and $h$, and applying the $1/\rho$
bound to each row gives
\[
\sum_h\lVert\delta A_h\rVert_F^2\le\frac{4s_{\max}^2N_\ell}{\rho^2}\lVert\delta Z\rVert_F^2,
\qquad
\sum_h\lVert\delta P_h\rVert_F^2\le\frac{s_{\max}^2N_\ell}{\rho^2}\lVert\delta Z\rVert_F^2,
\]
where the second inequality uses the softmax bound. Since
$\lVert V_h\rVert_2\le\lVert Z\rVert_F\le\sqrt{N_\ell}\,\nu$, the weight path
contributes at most $(s_{\max}N_\ell\nu/\rho)\lVert\delta Z\rVert_F$. The output
projection $W_O$ has norm at most one. Adding the two paths shows that the
sublayer has Jacobian norm at most $c_{\mathrm{sa}}$.

\emph{Step 3: Cosine cross-attention.}
The keys and values of cross-attention come from the fixed text tokens, so only
the queries vary with $X$. The argument of Step~2 then gives
$\lvert\delta A_{h,ij}\rvert\le s_{\max}\lVert\delta\widehat q_{h,i}\rVert_2$
for each of the $N_c$ keys $j$. Summing over keys, queries, and heads, and
applying the $1/\rho$ and softmax bounds, gives
$\sum_h\lVert\delta P_h\rVert_F^2\le\bigl(s_{\max}^2N_c/(4\rho^2)\bigr)\lVert\delta
Z\rVert_F^2$. The value matrices contribute the factor $\kappa_c$, and $W_O$
has norm at most one, so the sublayer has Jacobian norm at most
$c_{\mathrm{ca}}$. Selecting the patch-token queries and writing the output back
into the full token matrix are maps of norm one, so they leave this bound
unchanged.

\emph{Step 4: Composition.}
The Jacobian of a branch is the identity plus the product of three factors: the
channel-wise scale, the sublayer Jacobian, and the RMSNorm Jacobian. The first
factor has norm $\lVert\alpha\rVert_\infty$, and the last has norm at most
$c_{\mathrm n}$. By Steps~1--3, the self-attention and FFN branches of block $b$
therefore have Jacobian norms at most
$1+a^{(b)}_{\mathrm{sa}}c_{\mathrm n}c_{\mathrm{sa}}$ and
$1+a^{(b)}_{\mathrm{ffn}}c_{\mathrm n}c_{\mathrm{ffn}}$, and the
cross-attention branch has Jacobian norm at most
$1+a_{\mathrm{ca}}c_{\mathrm n}c_{\mathrm{ca}}$. The final RMSNorm contributes
$c_{\mathrm n}$. The linear logit head maps each token to one logit with a
vector of norm at most one, so it contributes at most one. The product of these
factors is $\Lambda_\ell$ in Eq.~(\plaineqref{eq:lip}).
\end{proof}

\noindent\textbf{Scope of the bound.}
Lemma~\ref{lem:lip} is conditional. The floor $\rho$ is an assumption about the
forward pass, and the architecture does not enforce it. The bound diverges as a
query or key approaches zero before normalization. The RMSNorm weights and the
residual scales are trainable, so $\Lambda_\ell$ changes during training
through $\bar w$, $a_{\mathrm{ca}}$, $a^{(b)}_{\mathrm{sa}}$, and
$a^{(b)}_{\mathrm{ffn}}$. The lemma treats $\lVert W\rVert_2\le1$ as
exact. Spectral normalization with finitely many power iterations satisfies
this condition only approximately, and any realized norm above one must be
included in the bound of the corresponding sublayer. The bound concerns the
head alone at a fixed prompt. Extending it to images requires, in addition,
Jacobian bounds for the frozen encoder and the current random mixing. The
constants are worst-case. For example, the RMSNorm Jacobian attains
$c_{\mathrm n}$ only near the zero token, and the factors $\sqrt{N_\ell}$ come
from Frobenius-norm accounting over all tokens. The lemma  shows how
spectral normalization, clamped cosine attention, RMSNorm, and residual scaling
each enter a finite sensitivity budget. 
Tab.~\ref{tab:head-ablation} evaluates the three controls of
Sec.~\ref{sec:head} jointly.
}

\section{Persistent Random Mixing: Analysis and Diagnostics}
\label{app:mixing}

This appendix supports Sec.~\ref{sec:refresh}. We describe the
implemented operator and its gain spectrum, make precise the sense in
which a direction is weakly measured under a draw, and isolate the
long-horizon consequence in a minimal measured-gap model. We report the empirical diagnostics in
Sec.~\ref{app:diagnostics}--\ref{app:probes} and state the limits of the account
in Sec.~\ref{app:mixing-scope}.

\subsection{The implemented random mixing and anisotropy}
\label{app:mixing-op}

\noindent\textbf{Random mixing in matrix form.}
	Eq.~(\plaineqref{eq:mixing}) composes CCM and CSM along a top-down path,
	but it leaves the two operators unspecified. We specify them here, and,
	because the mixing path carries no nonlinearity, we follow each one through to
	the matrix that Sec.~\ref{sec:refresh} reasons about. The construction proceeds in
	three steps: the matrix of one convolutional branch, then the matrix of one
	mixing, then the matrix through which the head of a given level reads.
	
	\emph{Setting.} For both $M^{\mathrm{cc}}_\ell$ and $M^{\mathrm{cs}}_\ell$, the
	patch tokens are reshaped to their two-dimensional grid. On a ViT backbone all
	$L$ taps share that grid, so every level carries the same number
	$N_{\mathrm p}$ of patch positions, each of token width $d$. The summary token
	is carried by its own random linear map, and it enters neither the matrices
	below nor the analysis that follows, so $f_\ell$ and $\hat f_\ell$ denote the
	patch part of the level-$\ell$ features throughout this appendix. We vectorize
	a level channel-major, so the features of one level form a vector of dimension
	$N_{\mathrm p}d$.
	
	\emph{Branch matrices.} Each operator passes the grid through three
	independent, size-preserving $d\to d$ convolutions, whose $1{\times}1$,
	$3{\times}3$, and $5{\times}5$ kernels are drawn by fan-in-scaled
	Kaiming-uniform initialization \citep{he2015delving}. Write $\mathcal M(W)$ for
	the matrix of one such convolution with kernel $W$. A $1{\times}1$ kernel
	$W^{(1)}\in\mathbb R^{d\times d}$ recombines channels and leaves positions
	untouched, so $\mathcal M(W^{(1)})=W^{(1)}\otimes I_{N_{\mathrm p}}$. A
	$k{\times}k$ kernel $W^{(k)}\in\mathbb R^{d\times d\times k\times k}$ gives the
	standard matrix of a multi-channel two-dimensional convolution
	\citep{araujo2021lipschitz,sedghi2019singular}. That matrix is a $d\times d$
	array of blocks, and its $(c',c)$ block is the matrix
	$\mathcal T\bigl(W^{(k)}_{c'c}\bigr)$ of order $N_{\mathrm p}$ generated by the
	$k{\times}k$ slice $W^{(k)}_{c'c}$ on the patch grid,
	\[
	\mathcal M\bigl(W^{(k)}\bigr)
	=\Bigl[\,\mathcal T\bigl(W^{(k)}_{c'c}\bigr)\,\Bigr]_{c',c=1}^{d}
	\in\mathbb R^{N_{\mathrm p}d\times N_{\mathrm p}d}.
	\]
	The convolutions zero-pad, so each block is banded doubly block-Toeplitz
	\citep{araujo2021lipschitz}. Circular padding would instead make each block
	doubly block-circulant and would block-diagonalize the operator by the
	two-dimensional DFT, which is the standard idealization of a zero-padded layer
	\citep{sedghi2019singular}. We invoke it only where the gain spectrum is
	needed, because decomposing the implemented operator directly is intractable at
	this width. Setting $k=1$ recovers
	$\mathcal T\bigl(W^{(1)}_{c'c}\bigr)=W^{(1)}_{c'c}I_{N_{\mathrm p}}$, for which
	the two paddings agree exactly.
	
	\emph{Mixing matrices.} Each mixing averages its three branches, so
	\[
	M^{\mathrm{cc}}_\ell
	=\tfrac13\!\!\sum_{k\in\{1,3,5\}}\!\!\mathcal M\bigl(W^{(k)}\bigr),
	\]
	and $M^{\mathrm{cs}}_\ell$ takes the same form from its own independent draw.
	All of these parameters are frozen after sampling, and they are reinitialized
	together at every refresh. Both matrices are square of order $N_{\mathrm p}d$,
	so each mixing preserves the backbone's channel width. The three branches share their stride and their
	sliding-window alignment, so this average is itself a single $5{\times}5$
	convolution. Zero-padding $W^{(1)}$ and $W^{(3)}$ to a $5{\times}5$ support and
	averaging the three kernels gives one kernel $\bar W$ with
	$\tfrac13\sum_k\mathcal M\bigl(W^{(k)}\bigr)=\mathcal M(\bar W)$
	\citep{ding2021dbb,ding2021repvgg}. The branches therefore do not enlarge the
	operator class, and what they change is the law from which the $5{\times}5$
	kernel is drawn. Fan-in scaling gives branch $k$ an entry variance proportional
	to $1/(dk^2)$, so the taps of $\bar W$ are center-heavy: the central tap
	receives all three branches, the remaining $3{\times}3$ taps receive two, and
	the outer $5{\times}5$ ring receives one.
	Mixing therefore stays predominantly a channel mixer, and it couples positions
	through a decaying spatial tail.
	
	\emph{The mixing matrix.} Stacking the levels, $f=[f_1^\top,\dots,f_L^\top]^\top$
	and $\hat f=[\hat f_1^\top,\dots,\hat f_L^\top]^\top$, the recursion of
	Eq.~(\plaineqref{eq:mixing}) unrolls into the affine map $\hat f=Mf+b$ used
	throughout the analysis. Its matrix is block upper triangular in the level
	index,
	\[
	M=\bigl[M_{\ell m}\bigr]_{\ell,m=1}^{L}\in\mathbb R^{d_f\times d_f},
	\quad
	M_{\ell m}=\begin{cases}
	\bigl(M^{\mathrm{cs}}_\ell M^{\mathrm{cs}}_{\ell+1}\cdots
	M^{\mathrm{cs}}_m\bigr)M^{\mathrm{cc}}_m, & m\ge\ell,\\[2pt]
	0, & m<\ell,
	\end{cases}
	\;\;
	d_f=LN_{\mathrm p}d .
	\]
	Every block is square of order $N_{\mathrm p}d$, so each level keeps the
	backbone's channel width. Two readings of this matrix matter later. Row
	$\ell$ vanishes to the left of the diagonal, which says that the head of
	level $\ell$ reads the backbone features of levels $\ell$ through $L$ and of
	those alone. The diagonal block
	$M_{\ell\ell}=M^{\mathrm{cs}}_\ell M^{\mathrm{cc}}_\ell$ is the CCM--CSM
	composition at that level, and it dominates column $\ell$, because every
	off-diagonal block reaching that column carries at least one further CSM
	factor. Since the diagonal blocks are almost surely invertible, so is $M$,
	and no feature direction is left unmeasured; what varies is how strongly each
	one is measured. 

	\noindent\textbf{Anisotropy.}
	
	 For a unit direction $\eta$, we define its gain as
	$\lambda_M(\eta)=\lVert M\eta\rVert$. In the circular-padding idealization
	of the token-grid convolutions above~\citep{sedghi2019singular}, these gains follow the
	singular values of square random matrices \citep{marchenko1967}. With the
	implemented initialization, a single CCM or CSM has a median gain of $0.27$,
	and ${\sim}5\%$ of directions lie below one tenth of that median. The
	diagonal block $M_{\ell\ell}$, which composes CCM with CSM, broadens the
	spectrum further, to a median of $0.063$ with ${\sim}12\%$ of directions
	below one tenth of the median; these blocks set the scale of the spectrum of
	$M$, whose off-diagonal blocks are smaller by a further CSM factor each. Zero padding
	truncates the kernels at the grid border, so the implemented operator has
	medians about $4\%$ and $9\%$ below these two values, and its tail fractions
	are unchanged. The exact weak directions depend on the draw. 
	The overall scale of the
	head gradient may change during training, but a fixed draw preserves the
	relative gain profile throughout the refresh interval.

\subsection{Gain-weighted correction under fixed and refreshed draws}
\label{app:mixing-model}

{\color{revtext}
		We isolate the effect of the mixing draw with a deliberately minimal
	surrogate. Let $\delta\in\mathbb{R}^{d_f}$ be a feature-space
	discrepancy between generated and real data, and let the generator
	reduce the measured gap $\frac12\lVert M\delta\rVert^2$ by the gradient
	flow $\dot\delta=-M^\top M\delta$. Write
	$M=\sum_i\lambda_i a_i\psi_i^\top$ for its singular value decomposition.
	Because $M$ is almost surely invertible (Sec.~\ref{app:mixing-op}), the
	$\psi_i$ form an orthonormal basis of $\mathbb{R}^{d_f}$ and every
	$\lambda_i$ is positive, so the expansions below are over a full basis.
	
	\begin{proposition}[Correction under a fixed draw and refreshed draws]
		\label{prop:gain-correction}
		\begin{enumerate}[label=\textup{(\roman*)},leftmargin=2.2em,
			itemsep=0.4em,topsep=0.3em]
			\item \textup{\textbf{Fixed draw.}}
			For a fixed $M$,
			\[
			\delta(t)=\sum_i e^{-\lambda_i^2t}
			\langle\psi_i,\delta(0)\rangle\psi_i.
			\]
			Hence, the component along $\psi_i$ decays at rate $\lambda_i^2$.
			
			\item \textup{\textbf{Independent refreshes.}}
			Suppose that independent draws $M^{(1)},\dots,M^{(K)}$ act for consecutive
			intervals of length $\tau$.
			Let $A^{(k)}=\exp\bigl(-\tau (M^{(k)})^\top M^{(k)}\bigr)$, and assume that their law is
			right-isotropic in the sense that
			$\mathbb E[(A^{(k)})^\top A^{(k)}]=c(\tau)I$. Then
			\[
			\mathbb E\lVert\delta(K\tau)\rVert^2
			=c(\tau)^K\lVert\delta(0)\rVert^2.
			\]
			Thus, in expectation, no fixed direction retains a privileged slow
			rate across refresh periods.
		\end{enumerate}
	\end{proposition}
	
	\begin{proof}
		\emph{Part (i).}
		Because $M^\top M=\sum_i\lambda_i^2\psi_i\psi_i^\top$, the
		coefficient of $\delta(t)$ along $\psi_i$ satisfies
		$\dot\delta_i(t)=-\lambda_i^2\delta_i(t)$. Solving each scalar
		equation gives the stated expansion.
		
		\emph{Part (ii).}
		By construction, $\delta(K\tau)=A^{(K)}\cdots A^{(1)}\delta(0)$. Condition on
		$y=A^{(K-1)}\cdots A^{(1)}\delta(0)$. Independence and the
		right-isotropy assumption give
		$\mathbb E[\lVert A^{(K)}y\rVert^2\mid y]
		=y^\top\mathbb E[(A^{(K)})^\top A^{(K)}]y
		=c(\tau)\lVert y\rVert^2$. Iterating this identity over the $K$
		periods proves the result.
	\end{proof}
	
	Part (i) formalizes the fixed-draw selection described in
	Sec.~\ref{sec:refresh}. High-gain components decay first, so the
	surviving residual becomes enriched in weakly measured directions.
	The measured gap can therefore settle while a feature-space residual
	remains. Weak measurement here is a matter of degree rather than of kind: no
	direction is annihilated, because $M$ is almost surely invertible, but a
	direction with a small $\lambda_i$ is corrected at the far smaller rate
	$\lambda_i^2$ and can therefore survive an entire fixed-draw run.
	
		Part (ii) isolates the effect of refresh. Each independent draw replaces
	the gain landscape, so in expectation no fixed direction is
	systematically attenuated over the whole run. The implemented
	convolutional law has exact permutation and sign symmetries over channel
	coordinates, and every level is read chiefly through the diagonal block that
	dominates its column, so the expected gain is comparable across levels. Right
	isotropy over the full feature space is the corresponding idealization, and we
	use it only to make this temporal distinction precise.

	The surrogate predicts differential correction and stall. 
	The nonlinear claim that realization-specific co-adaptation leaves new
	discrepancies uncorrected in weak directions and causes FID oscillation
	or degradation therefore remains a hypothesis. The refresh intervention
	in Sec.~\ref{sec:refresh} validates this hypothesis.
}

\subsection{Diagnostics}
\label{app:diagnostics}

\paragraph{Definitions.}
{\color{revtext}
For an image $x$, let $f_\ell(x)\in\mathbb{R}^{N_\ell\times d}$ denote the
level-$\ell$ backbone features after register-token handling and before mixing,
and let $\bar D(x)=\frac1n\sum_iD_i(x)$ be the mean critic logit. With the
encoder detached, the gradient $\nabla_{f_\ell}\bar D(x)$ is the critic
sensitivity in the shared backbone coordinates. We define the
\emph{Pre-projection Feature Utilization Coverage} (PFUC) profile by
\[
u_c(x)=\sigma_c\sqrt{\tfrac{1}{N_\ell}\textstyle\sum_{j=1}^{N_\ell}
\bigl(\partial\bar D(x)/\partial f_{\ell,j,c}(x)\bigr)^2},\qquad
U_c=\mathbb{E}_x[u_c(x)],\qquad
\pi_c=U_c\Big/\textstyle\sum_{c'}U_{c'},
\]
where $j$ indexes tokens, $\sigma_c$ is the feature scale of channel $c$,
computed once on a reference set and shared across runs, and the expectation is over a
fixed diagnostic set of real and generated images. We quantify the
change between two PFUC profiles by their total-variation distance,
\[
\mathrm{TV}(\pi,\pi')=\tfrac12\textstyle\sum_c
\lvert\pi_c-\pi'_c\rvert,
\]
which is the fraction of normalized sensitivity mass reassigned across
channels. 
PFUC characterizes the breadth and stability of the critic's input-gradient
sensitivity \citep{simonyan2014saliency}. 

The \emph{logit margin} is $m_t=r_t-q_t$, where $r_t$ and $q_t$ are the
batch-mean values of $\bar D$ on the real and generated halves of the step-$t$
batch, respectively. We apply $\pm150$-step smoothing only for display. Unlike
accuracy, $m_t$ retains the magnitude of the real--generated separation, and the
separate traces of $r_t$ and $q_t$ reveal how that separation changes.
}

\paragraph{Measured values (Figs.~\ref{fig:fixed-diagnosis}
and~\ref{fig:refresh-mechanism}).}
{\color{revtext}
The two diagnostic runs use SDXL with an effective batch size of $32$ and
share the seeds, data, diagnostic set, and channel feature scales. They differ
only in whether $M$ is fixed or re-drawn every $5$k steps.
The PFUC heat maps display the first $64$ of $768$ channels at the
deepest tap on a shared logarithmic color scale, with one profile every
$500$ steps; all TV values use the complete $768$-channel profiles. In
the refresh run, an additional profile is recorded immediately after
each re-draw. 
}

\subsection{Longitudinal probes on a fixed-mixing run}
\label{app:probes}

{\color{revtext}
		The fixed-run measurements in Sec.~\ref{app:diagnostics} establish a
	persistent critic state, but they do not determine whether its separation
	is specific to the training realization $M^{(0)}$. We therefore apply two
	post-hoc diagnostics to frozen checkpoints of the same run.
	
	\emph{Fresh-draw transfer.}
	At every $5$k-step checkpoint from $5$k to $30$k, we hold the checkpoint,
	critic head, and matched real and generated batches fixed, and replace
	only $M^{(0)}$ by eight independently drawn mixings. The margin under $M^{(0)}$
	remains $5.1$--$5.3$ through $25$k ($4.5$ at $30$k), whereas the margins
		under fresh draws are approximately zero at every checkpoint. Because
	only the mixing changes, this collapse shows that the fixed head's
	real--generated separation does not transfer to a fresh measurement
	basis. This result directly supports its coupling to $M^{(0)}$.
	
	\emph{Linear probe.}
	We also fit a separate linear classifier to the frozen mixed features at
	each level to distinguish real from generated samples. This classifier
		is used only as a diagnostic and does not update the generator, backbone,
	mixing, or critic. Beginning at approximately $2.5$k steps, real and
	generated samples are harder for a linear classifier to distinguish under
	$M^{(0)}$ than under fresh draws, particularly for the shallow-level
	features.

	Over the same fixed-mixing run, FID oscillates and later worsens, whereas
	the margin under the fixed training mixing $M^{(0)}$ remains large and nearly
	stationary. The fresh-draw and linear-probe diagnostics show that this
	separation is specific to the pairing $D_\phi\circ M^{(0)}$.
}

\subsection{Scope of the account}
\label{app:mixing-scope}

The account makes two claims at different strengths. First,
Prop.~\ref{prop:gain-correction} shows in a linear surrogate why
weak-direction residuals persist under a fixed draw and why independent draws
remove a permanently privileged direction in expectation. The chain-rule bounds
of Sec.~\ref{sec:refresh} tie the visibility and correction of a direction to its
gain. Second, the realization-specific co-adaptation of Sec.~\ref{sec:refresh}, in
which prolonged optimization under the fixed gain landscape hardens both
players around it until quality degrades, is a mechanistic reading of the
ongoing nonlinear game. It is supported jointly by PFUC, the margin, the
fresh-draw probes of Sec.~\ref{app:probes}, and the fixed-versus-refreshed
intervention, but none of these alone establishes it.

\section{\texorpdfstring{Jacobian R1: Analysis and Estimation}{Jacobian R1: Analysis and Estimation}}
\label{app:jr1}

This appendix supports Sec.~\ref{sec:jr1}. We compare pooled R1 and JR1 in the
output space of the critic in Sec.~\ref{app:jr1-pooled}, analyze the Rademacher
estimator of JR1 in Sec.~\ref{app:jr1-estimator}, and specify its minibatch
implementation in Sec.~\ref{app:jr1-batch}.

All quantities are evaluated at a fixed prompt $c$, critic parameters $\phi$,
and mixing $M$. We use the notation of Sec.~\ref{sec:jr1}, where
$J(x)\in\mathbb R^{n\times d_x}$ is the input--logit Jacobian with rows
$g_i(x)^\top$, and $d_x$ is the dimension of the image $x$. Pooled R1 and JR1
in Eqs.~(\plaineqref{eq:jr1-pool}) and~(\plaineqref{eq:jr1}) are the
expectations of the per-image energies $q_{\mathrm{pool}}(x)$ and
$q_{\mathrm{Jac}}(x)$ over real images, scaled by $\gamma/2$ and
$\gamma_{\mathrm J}/2$. The Jacobian is taken through the entire critic,
including the frozen encoder $F$ and the mixing $M$. Freezing their parameters
excludes them from optimization, but it does not remove them from the path of
the input gradient. Because both energies depend on $x$ only through $J(x)$, we
also evaluate them at an arbitrary matrix $J\in\mathbb R^{n\times d_x}$ and
write $q_{\mathrm{pool}}(J)$ and $q_{\mathrm{Jac}}(J)$. Both energies are built
from averages over the $n$ logits. We call this $1/n$ normalization the
averaging convention, and the estimator and the minibatch penalty in
Sec.~\ref{app:jr1-estimator}--\ref{app:jr1-batch} retain it.

\subsection{\texorpdfstring{Jacobian R1 versus pooled R1}{Jacobian R1 versus pooled R1}}
\label{app:jr1-pooled}

Sec.~\ref{sec:jr1} makes three claims about pooled R1. First, pooled
R1 penalizes only the direction spanned by $\mathbf1$ in the
$n$-dimensional output space, and it leaves the $n-1$ orthogonal directions
unpenalized. Second, a nonzero Jacobian can therefore incur zero pooled
penalty. Third, the dense loss of Eq.~(\plaineqref{eq:dense}) still depends on
the sensitivities that pooling neglects. Parts~(i)--(iii) of
Prop.~\ref{prop:pooling} establish these claims in turn, and each part
also states the corresponding property of JR1. 

Fix an image $x$, and abbreviate $J=J(x)$ and $g_i=g_i(x)$. We use the
orthogonal projectors onto $\operatorname{span}\{\mathbf1\}$ and onto its
$(n-1)$-dimensional orthogonal complement $\mathbf1^{\perp}$, together with the
mean gradient:
\[
\Pi_{\parallel}=\tfrac1n\mathbf1\mathbf1^\top,
\qquad
\Pi_{\perp}=I_n-\Pi_{\parallel},
\qquad
\bar g=\tfrac1n\textstyle\sum_{i=1}^n g_i=\tfrac1nJ^\top\mathbf1 .
\]
The $r$-th column $j_r=\partial D(x)/\partial x_r$ of $J$ is the response of
the logit vector to the $r$-th input coordinate. We call
\[
q_{\perp}(x)=\tfrac1n\lVert\Pi_{\perp}J\rVert_F^2
=\tfrac1n\textstyle\sum_{r=1}^{d_x}\lVert\Pi_{\perp}j_r\rVert_2^2
\]
the \emph{neglected energy}. It measures how strongly input perturbations move
the logit vector within the $n-1$ directions orthogonal to $\mathbf1$.

\begin{proposition}[Pooled R1 and JR1 in output space]
\label{prop:pooling}
\begin{enumerate}[label=\textup{(\roman*)},leftmargin=2.2em,itemsep=0.4em,topsep=0.3em]
\item \textup{\textbf{Decomposition.}}
The pooled energy satisfies
$q_{\mathrm{pool}}(x)=\tfrac1n\lVert\Pi_{\parallel}J\rVert_F^2$, and
\begin{equation}
q_{\mathrm{Jac}}(x)=q_{\mathrm{pool}}(x)+q_{\perp}(x),
\qquad
q_{\perp}(x)=\tfrac1n\textstyle\sum_{i=1}^n\lVert g_i-\bar g\rVert_2^2 .
\label{eq:dispersion}
\end{equation}
\item \textup{\textbf{Zero set.}}
The pooled energy $q_{\mathrm{pool}}(J)$ vanishes exactly on the
$(n-1)d_x$-dimensional subspace
$\mathcal Z=\{J\in\mathbb R^{n\times d_x}:J^\top\mathbf1=0\}$. In contrast,
$q_{\mathrm{Jac}}(J)$ vanishes only at $J=0$.
\item \textup{\textbf{Dense-loss gradient.}}
Let $\ell_{\mathrm{real}}(x)=\tfrac1n\sum_{i=1}^n\zeta(-D_i(x))$ be the
per-image real-image term of $\mathcal L_D$ in
Eq.~(\plaineqref{eq:dense}). Then
\begin{equation}
\lVert\nabla_x\ell_{\mathrm{real}}(x)\rVert_2^2\le q_{\mathrm{Jac}}(x).
\label{eq:jr1-loss-bound}
\end{equation}
In contrast, $q_{\mathrm{pool}}$ admits no such bound. If the logits $D_i(x)$
are not all equal, then $\lVert\nabla_x\ell_{\mathrm{real}}(x)\rVert_2$ is
unbounded as $J(x)$ ranges over $\mathcal Z$ with $D(x)$ held fixed.
\end{enumerate}
\end{proposition}

\begin{proof}
\emph{Part (i).}
Since $\Pi_{\parallel}J=\mathbf1\bar g^\top$, every row of $\Pi_{\parallel}J$
equals $\bar g^\top$, and the $i$-th row of $\Pi_{\perp}J=J-\Pi_{\parallel}J$
equals $(g_i-\bar g)^\top$. Hence
$\tfrac1n\lVert\Pi_{\parallel}J\rVert_F^2=\lVert\bar g\rVert_2^2=q_{\mathrm{pool}}(x)$,
and $q_{\perp}(x)=\tfrac1n\sum_i\lVert g_i-\bar g\rVert_2^2$. The deviations
$g_i-\bar g$ sum to zero. Averaging
$\lVert g_i\rVert_2^2=\lVert\bar g+(g_i-\bar g)\rVert_2^2$ over $i$ therefore
cancels the cross terms and gives
$q_{\mathrm{Jac}}(x)=\lVert\bar g\rVert_2^2+q_{\perp}(x)$.

\emph{Part (ii).}
By definition, $q_{\mathrm{pool}}(J)=\tfrac1{n^2}\lVert J^\top\mathbf1\rVert_2^2$
vanishes if and only if $J^\top\mathbf1=0$. This condition places each of the
$d_x$ columns of $J$ in the $(n-1)$-dimensional subspace $\mathbf1^{\perp}$, so
$\dim\mathcal Z=(n-1)d_x$. The energy
$q_{\mathrm{Jac}}(J)=\tfrac1n\lVert J\rVert_F^2$ is a scaled squared norm, so it
vanishes only at $J=0$.

\emph{Part (iii).}
Let $\eta\in\mathbb R^n$ collect the logit derivatives
$\eta_i=\partial\ell_{\mathrm{real}}/\partial D_i=-\tfrac1n\zeta'(-D_i(x))$. The
chain rule gives $\nabla_x\ell_{\mathrm{real}}(x)=J^\top\eta$. The derivative
$\zeta'$ is the logistic function, which takes values in $(0,1)$, so
$\lVert\eta\rVert_2^2<1/n$. The spectral norm of $J$ is at most its Frobenius
norm. Therefore
$\lVert J^\top\eta\rVert_2^2\le\lVert J\rVert_F^2\lVert\eta\rVert_2^2
\le\tfrac1n\lVert J\rVert_F^2=q_{\mathrm{Jac}}(x)$, which is
Eq.~(\plaineqref{eq:jr1-loss-bound}).
For the second statement, note that $\eta$ depends only on $D(x)$, which is
held fixed. The logistic function is strictly increasing. If the logits are not
all equal, the entries of $\eta$ are therefore not all equal, and
$\eta_{\perp}=\Pi_{\perp}\eta\ne0$. Fix a unit vector $a\in\mathbb R^{d_x}$, and
let $J_t=t\,\eta_{\perp}a^\top$ for $t>0$. Because
$\mathbf1^\top\eta_{\perp}=0$, we have $J_t^\top\mathbf1=0$, and hence
$J_t\in\mathcal Z$. Because $\eta_{\perp}^\top\eta=\lVert\eta_{\perp}\rVert_2^2$,
we also have $J_t^\top\eta=t\lVert\eta_{\perp}\rVert_2^2\,a$, whose norm grows
without bound as $t\to\infty$.
\end{proof}

\subsection{Rademacher estimator: unbiasedness and variance}
\label{app:jr1-estimator}

Let $v_1,\dots,v_P\in\{-1,+1\}^n$ be independent probes, and let each probe
have independent, equiprobable signs, as in Sec.~\ref{sec:jr1}. The per-image
estimator replaces $\mathbb E_v$ in Eq.~(\plaineqref{eq:jr1-hutchinson}) by an
average over the $P$ probes:
\[
\widehat q_P(x)=\frac1{nP}\sum_{p=1}^P\lVert J(x)^\top v_p\rVert_2^2 .
\]
Its factor $1/n$ follows the averaging convention, so $\widehat q_P(x)$
estimates $q_{\mathrm{Jac}}(x)$. This is the trace estimator of
\citet{hutchinson1989} for $\operatorname{tr}\bigl(J(x)J(x)^\top\bigr)
=\lVert J(x)\rVert_F^2$, which earlier work also used for Jacobian
regularization \citep{varga2018gradient,hoffman2019jacobian}.
Prop.~\ref{prop:jr1} restates its standard properties under the
averaging convention.

\begin{proposition}[Unbiasedness and variance]
\label{prop:jr1}
Fix an image $x$, and write $g_i=g_i(x)$.
\begin{enumerate}[label=\textup{(\roman*)},leftmargin=2.2em,itemsep=0.4em,topsep=0.3em]
\item \textup{\textbf{Unbiasedness.}}
$\mathbb E_v[\widehat q_P(x)]=q_{\mathrm{Jac}}(x)$. If $J(x)$ is
differentiable in $\phi$, then also
$\mathbb E_v[\nabla_\phi\widehat q_P(x)]=\nabla_\phi q_{\mathrm{Jac}}(x)$.
\item \textup{\textbf{Variance.}}
The variance over the probes is
\[
\operatorname{Var}_v[\widehat q_P(x)]
=\frac{2}{n^2P}\sum_{i\ne j}\bigl(g_i^\top g_j\bigr)^2 .
\]
It vanishes if and only if the per-logit gradients are pairwise orthogonal.
\item \textup{\textbf{Relative variance.}}
If $J(x)\ne0$, then
$\operatorname{Var}_v[\widehat q_P(x)]\le\frac2P\bigl(1-\frac1n\bigr)\,q_{\mathrm{Jac}}(x)^2$.
Equality holds when every $g_i$ equals $\pm g$ for a common vector $g\ne0$.
\end{enumerate}
\end{proposition}

\begin{proof}
For a single probe $v$, the identity $v_i^2=1$ gives
\[
\lVert J(x)^\top v\rVert_2^2
=\sum_{i,j}v_iv_j\,g_i^\top g_j
=\sum_i\lVert g_i\rVert_2^2+2\sum_{i<j}v_iv_j\,g_i^\top g_j .
\]
Each product $v_iv_j$ with $i<j$ has zero mean and unit variance, and distinct
products are uncorrelated.

\emph{Part (i).}
The cross terms have zero mean, so
$\mathbb E_v\lVert J(x)^\top v\rVert_2^2=\sum_i\lVert g_i\rVert_2^2=n\,q_{\mathrm{Jac}}(x)$.
Averaging over the $P$ probes and dividing by $n$ gives the first identity.
The probe law is uniform on the finite set $\{-1,+1\}^n$ and does not depend
on $\phi$. The expectation $\mathbb E_v$ is therefore a finite average, which
commutes with $\nabla_\phi$.

\emph{Part (ii).}
The cross terms are uncorrelated and have unit variance. A single probe
therefore gives
$\operatorname{Var}_v\lVert J(x)^\top v\rVert_2^2=4\sum_{i<j}(g_i^\top g_j)^2
=2\sum_{i\ne j}(g_i^\top g_j)^2$. The $P$ probes are independent, so averaging
over them divides this variance by $P$. The factor $1/n$ divides it by $n^2$.

\emph{Part (iii).}
The Cauchy--Schwarz inequality gives
$(g_i^\top g_j)^2\le\lVert g_i\rVert_2^2\lVert g_j\rVert_2^2$ for every pair,
and it also gives
$\sum_i\lVert g_i\rVert_2^4\ge\frac1n\bigl(\sum_i\lVert g_i\rVert_2^2\bigr)^2$.
Hence
\[
\sum_{i\ne j}(g_i^\top g_j)^2
\le\Bigl(\sum_i\lVert g_i\rVert_2^2\Bigr)^2-\sum_i\lVert g_i\rVert_2^4
\le\Bigl(1-\frac1n\Bigr)\Bigl(\sum_i\lVert g_i\rVert_2^2\Bigr)^2
=\Bigl(1-\frac1n\Bigr)n^2q_{\mathrm{Jac}}(x)^2 .
\]
Substituting this bound into part~(ii) gives the claim. If every $g_i$ equals
$\pm g$, both inequalities hold with equality.
\end{proof}

\noindent\textbf{Choice of Rademacher probes.}
Gaussian probes $v\sim\mathcal N(0,I_n)$ also satisfy $\mathbb E[vv^\top]=I_n$,
so they also give an unbiased estimator. For Gaussian probes, however, the
diagonal terms $v_i^2\lVert g_i\rVert_2^2$ fluctuate. At equal probe count, they
add $\frac{2}{n^2P}\sum_i\lVert g_i\rVert_2^4$ to the variance of part~(ii).
Rademacher signs remove this term because $v_i^2=1$ holds exactly.

\subsection{Minibatch implementation and normalization}
\label{app:jr1-batch}

Consider a minibatch $\mathcal B$ of real image--prompt pairs $(x,c)$. For each
pair, we draw $P$ probes $v_{x,1},\dots,v_{x,P}$ with the law of
Sec.~\ref{app:jr1-estimator}. The probes are independent across pairs, and they
are drawn afresh at every evaluation of the penalty. The minibatch penalty is
\begin{equation}
\widehat R_1^{\mathrm J}
=\frac{\gamma_{\mathrm J}}{2|\mathcal B|}\sum_{(x,c)\in\mathcal B}\widehat q_P(x)
=\frac{\gamma_{\mathrm J}}{2n|\mathcal B|P}\sum_{(x,c)\in\mathcal B}\sum_{p=1}^{P}
\bigl\lVert\nabla_x\bigl(v_{x,p}^\top D(x)\bigr)\bigr\rVert_2^2 ,
\label{eq:jr1-batch}
\end{equation}
where $D(x)$ is evaluated at the prompt $c$ paired with $x$. The second form
uses the identity $J(x)^\top v=\nabla_x\bigl(v^\top D(x)\bigr)$. By
Prop.~\ref{prop:jr1}(i), the expectation of $\widehat R_1^{\mathrm J}$
over the probes and over minibatches drawn from $p_r$ is $R_1^{\mathrm J}$ of
Eq.~(\plaineqref{eq:jr1}). The same holds for its gradient with respect to
$\phi$. The probes are independent across
images, so the probe noise in Eq.~(\plaineqref{eq:jr1-batch}) is averaged over
the $|\mathcal B|$ images as well as over the $P$ probes.

\noindent\textbf{Cost and order of operations.}
Each term of Eq.~(\plaineqref{eq:jr1-batch}) is one vector--Jacobian product
(VJP). A single backward pass of the scalar $v_{x,p}^\top D(x)$ through the
critic returns $J(x)^\top v_{x,p}$. The estimator therefore costs $P$ VJPs per
image, whereas the exact penalty costs $n$ VJPs per image (Sec.~\ref{sec:jr1}).
Pooled R1 in Eq.~(\plaineqref{eq:r1}) backpropagates the scalar
$\frac1n\mathbf1^\top D(x)$ in the same way. With $P=1$, the two penalties
therefore have the same cost, and they differ only in the weights of the logit
sum. Each probe gradient must be squared before the average over probes is taken.
Averaging the $P$ probe gradients first would give an estimate with expectation
$q_{\mathrm{Jac}}(x)/P$.

\noindent\textbf{Lazy scaling and reported coefficients.}
We apply the lazy regularization of Sec.~\ref{sec:gp}
\citep{karras2020stylegan2}. On every $k_{\mathrm{lazy}}$-th critic update, the
critic objective $\mathcal L_D+R_D$ uses $R_D=k_{\mathrm{lazy}}\widehat
R_1^{\mathrm J}$. On all other critic updates, $R_D=0$. Both configurations in Tab.~\ref{tab:hparams} use $P=1$ and
$k_{\mathrm{lazy}}=16$. The table reports $\gamma_{\mathrm J}$ as the
coefficient of the averaged penalty in Eq.~(\plaineqref{eq:jr1}), before the
factor $k_{\mathrm{lazy}}$ is applied. The reported values,
\scopechange{$\gamma_{\mathrm J}=7.5\times10^{-3}$ for SDXL and
$\gamma_{\mathrm J}=7.5\times10^{-4}$ for SD3.5-M}, therefore require no
conversion.

\section{Related Work}
\label{app:related}

\providecommand{\rwsub}[1]{\textbf{\emph{#1}}}

\noindent\textbf{Few-step distillation of text-to-image diffusion models.}	
\rwsub{Multi-objective or multi-stage training.} The existing methods mostly use multiple training objectives and often consist of multiple training stages. Hyper-SD employs a six-phase hybrid training pipeline: four stages of trajectory-segmented consistency distillation, followed by human-feedback optimization and DMD-based one-step score distillation \citep{ren2024hypersd}. For the one-step SDXL model, DMD2 uses a two-stage pipeline consisting of brief regression pretraining followed by joint distribution-matching and adversarial training stabilized by a two-timescale update rule \citep[Sec. F.4]{yin2024dmd2}; DMDX is also a two-stage pipeline, consisting of adversarial distillation pretraining on teacher-generated ODE pairs with hybrid latent- and pixel-space discriminators, along with adversarial distribution-matching fine-tuning \citep{lu2025adm}. SiDA jointly trains a one-step generator with score-identity and adversarial losses, reusing the fake score network’s encoder as a parameter-free discriminator \citep{zhou2024sida}. ADD (SDXL-Turbo) uses single-stage joint training with teacher-based score distillation and a DINOv2-feature adversarial loss~\citep{sauer2024add}. 
SDXL-Lightning employs a five-stage progressive distillation pipeline: one MSE-based trajectory-distillation stage followed by four adversarial stages \citep{lin2024sdxllightning}. NitroFusion uses dynamic adversarial refinement with a pool of conditional and unconditional global–local discriminators, \emph{relying on a DMD2 one-step student and four-step teacher \citep{chen2025nitrofusion}, so separate score-based pretraining is a prerequisite}. SenseFlow combines score-based distribution matching, trajectory-based intra-segment guidance, and timestep-aware VFM adversarial supervision \citep{ge2026senseflow}. SwD produces a progressive few-step generator that increases latent resolution across denoising steps, combining patch-level MMD with score-based distribution matching and adversarial loss \citep{swd2025}.
\rwsub{Adversarial objective as a single loss.} Both APT \citep{lin2025apt} and LADD \citep{sauer2024ladd} perform latent-space adversarial training and derive \textit{multilevel discriminator features from pretrained diffusion models}. 
LADD freezes the diffusion teacher, attaches discriminator heads to its intermediate features, and uses teacher-generated samples as real targets. APT initializes a separate discriminator from the pretrained diffusion model, adds multilevel output heads, and fine-tunes it against real data. D2O~\citep{zheng2025d2o} converts a pretrained diffusion U-Net into a one-step generator through pure adversarial fine-tuning; it is evaluated \textit{only on low-resolution unconditional and class-conditional benchmarks and does not address text-to-image synthesis}. Unlike APT and LADD, both of which use diffusion-based critics—trainable in APT and frozen in LADD—ROAD uses a frozen foundation model (e.g., DINOv3) with lightweight, prompt-conditioned Transformer heads. Compared with D2O, we propose Transformer heads, refreshed mixing, and JR1, thereby scaling pure adversarial training to large-scale text-to-image generation.

\noindent\textbf{Projected GAN and adjacent critic architectures.}
\rwsub{Convolutional heads over frozen encoders.} Projected GAN trains shallow
convolutional heads over a frozen pretrained encoder, whose features are scattered by fixed random
cross-channel and cross-scale mixings \citep{sauer2021projected}; StyleGAN-XL scales
it to ImageNet \citep{sauer2022styleganxl}, inheriting Projected GAN's heads, and StyleGAN-T to text-to-image with
1-D convolutional heads over frozen ViT tokens \citep{sauer2023stylegant}.
Every head in this line is convolutional, the only reported head comparison
being 1-D versus 2-D convolution \citep{sauer2023stylegant}, and positions are
scored independently, so the critic is vector-valued, a property ROAD keeps and
regularizes. As far as we know, no work in this line
attaches a Transformer head to frozen features.
\rwsub{Transformer discriminators from scratch.} ViTGAN scores a ViT on
raw pixels from the class token, stabilized by $L_2$ attention with tied query
and key projections, and applies spectral normalization only to initialize the
linear weights \citep{lee2022vitgan}. It reports that the vanilla ViT discriminator trains
markedly worse than the stabilized one. TransGAN's convolution-free discriminator ingests
patches of several sizes \citep{jiang2021transgan}. 
Other work pairs Transformer generators with convolutional discriminators or
embeds attention inside a CNN
\citep{zhang2022styleswin,zhao2021hit,park2022styleformer,hudson2021ganformer}.
Our Lipschitz controls build on analyses of
attention \citep{kim2021lipschitz,qi2023lipsformer}, scaled cosine attention
\citep{liu2022swinv2}, and spectral normalization \citep{miyato2018sn}; their
combination into the projected critic's ViT head with an explicit bound (Lemma~\ref{lem:lip}) and its conditioning on individual prompt
tokens are new.

\noindent\textbf{Random feature mixing and time-varying discriminators.}
\rwsub{Fixed random mixing.} Projected GAN samples the mixing once at
initialization and holds it fixed for the whole run
\citep{sauer2021projected}, and every method in that line inherits the choice;
the consequence of one draw persisting through long training, a critic that
specializes to it, is, to our knowledge, unexamined. Previously, multiple-random-projection GANs train
simultaneous discriminators on a fixed ensemble of input projections
\citep{neyshabur2017mrp}, as do random-projection ensembles for classification
\citep{cannings2017random}; in both, the set of views is fixed in advance.
\rwsub{Renewing trainable parameters.} The closest precedents do alter the
discriminator over time. NitroFusion keeps a pool of trainable heads, samples a
subset per iteration, and re-initializes a small fraction of it, and SD3.5-Flash
follows with stochastic head renewal
\citep{chen2025nitrofusion,bandyopadhyay2025sd35flash}; SDXL-Lightning resets
its discriminator between progressive-distillation stages rather than within
one \citep{lin2024sdxllightning}.
\emph{Outside GANs,} re-initializing trainable parameters is
the standard remedy for lost plasticity, through last-layer resets, recycling
of dormant neurons, shrink-and-perturb, and continual backpropagation
\citep{nikishin2022primacy,sokar2023dormant,ash2020warm,dohare2024plasticity}.
Our ROAD instead re-draws the frozen measurement basis, the random mixings, at
deterministic intervals while every trainable parameter
persists, so it targets not lost plasticity in the head but the head's
specialization to a single persistent realization. 

\noindent\textbf{Discriminator gradient regularization.}
\rwsub{Zero-centered penalties in GANs.} Zero-centered discriminator-gradient
penalties stabilize GAN training and admit local convergence analyses under
suitable assumptions \citep{roth2017stabilizing,mescheder2018r1}, and R3GAN
combines R1 and R2 \citep{huang2024r3gan}. D2O applies
pooled R1 to the averaged logit \citep{zheng2025d2o}, ADD and P2D regularize in
feature space \citep{sauer2024add,chong2024p2d}, and APT uses a
finite-difference approximation \citep{lin2025apt}. Each constrains a scalar
critic output, or an aggregate of a vector-valued one, which is the assumption our 
ROAD removes.
\rwsub{Jacobian regularization in classification.}
\citet{novak2018sensitivity} study the empirical relationship between
input--output Jacobian sensitivity and generalization, whereas
\citet{jakubovitz2018jacobian} explicitly regularize the logit Jacobian to
improve adversarial robustness, and
\citet{varga2018gradient,hoffman2019jacobian} develop efficient randomized
Jacobian regularization. We apply a squared Jacobian-norm
penalty to real images in a dense adversarial game, motivated specifically by the limitation of pooled R1, which aggregates the vector-valued logit
(Sec.~\ref{sec:jr1}). Through Rademacher Hutchinson estimation \citep{hutchinson1989} we 
recover the Jacobian energy in expectation with one input
VJP per probe.

\needspace{12\baselineskip}
\section{Implementation Details}
\label{app:more-exp}

\subsection{Critic design}
\label{app:road-head}

{\color{revtext}%
This subsection specifies the critic heads of Sec.~\ref{sec:head}. The critic has
one head per feature level, so it has $L=4$ heads (Sec.~\ref{sec:prelim}). Head
$\ell$ reads the mixed tokens $\hat f_\ell$ of Eq.~(\plaineqref{eq:mixing}),
which the random mixing computes from the features of the frozen encoder $F$
(Sec.~\ref{app:mixing-op}).

\noindent\textbf{Level and depth schedules.}
Let $L_F$ be the number of blocks of the frozen encoder, indexed from $0$ to
$L_F-1$. The features $f_\ell$ of Eq.~(\plaineqref{eq:mixing}) are the tokens
output by block $t_\ell$ of $F$, and the head of level $\ell$ has $B_\ell$
Transformer blocks, where
\[
t_\ell=\Bigl\lfloor\frac{\ell L_F}{L}\Bigr\rfloor-1,
\qquad
B_\ell=\min\bigl(L-\ell+1,\,4\bigr),
\qquad
\ell=1,\dots,L.
\]
The first rule divides $F$ into $L$ contiguous spans of equal length and taps
the last block of each span, so the deepest level always taps the final block,
$t_L=L_F-1$. The second rule gives the deepest level a single block, because
its tokens are already highly semantic and need only light adaptation. The
depth then increases by one per level toward the shallowest level, which
receives the most reorganization. The cap of four blocks keeps the heads
shallow and limits the number of factors in the Lipschitz bound of
Eq.~(\plaineqref{eq:lip}). For the configuration used throughout the paper,
DINOv3 ViT-B with $L_F=12$, the taps are $(2,5,8,11)$, and the depths are
$(B_1,\dots,B_4)=(4,3,2,1)$, which gives ten trainable blocks in total. Applied
to a $24$-block ViT-L, the same rule gives the taps $(5,11,17,23)$. The depths
depend only on $L$, so they are unchanged.

\noindent\textbf{Transformer block.}
A block follows the pre-norm Transformer design. It contains a self-attention
branch followed by an FFN branch, and the first block of every head contains a
cross-attention branch between the two. The block departs from the standard
design in four ways. First, RMSNorm replaces LayerNorm. Second, queries and keys
have width $d/2$, where $d$ is the token width. Third, the FFN uses GELU with an
expansion ratio of $2$ instead of $4$. Fourth, self- and cross-attention use
cosine attention with a clamped temperature. All learned linear maps of the head
are spectrally normalized (Sec.~\ref{sec:head}), and token positions are encoded
with rotary position embeddings (Tab.~\ref{tab:hparams}).
Sec.~\ref{app:lipschitz} analyzes the resulting head.

\noindent\textbf{Cross-attention.}
The cross-attention branch has the same cosine form. Its queries are the patch
tokens of the level, and its keys and values are computed from the prompt. The
summary token is excluded from the queries, so conditioning acts on the patch
tokens only. The token sequences of all text encoders are concatenated along the
channel axis into a single sequence and normalized without affine parameters.
Two spectrally normalized linear maps then project this sequence to the widths
of the visual queries and values. Each level has its own pair of text
projections, so different levels can read the same prompt differently.

\noindent\textbf{Attention temperature.}
Each attention head $h$ has a learnable temperature $s_h$. It is stored in the
log domain, initialized at $\log s_{\mathrm{init}}$, and clamped at
$\log s_{\max}$ before exponentiation, with $s_{\mathrm{init}}=10$ and
$s_{\max}=100$. This parametrization guarantees $0<s_h\le s_{\max}$. Queries and
keys are $\ell_2$-normalized, the normalized queries are multiplied by $s_h$,
and the result is passed to scaled dot-product attention. Normalization,
scaling, and attention are computed in single precision and cast back to the
training precision. This keeps the clamp exact under mixed-precision training.

\noindent\textbf{Residual scales and gate.}
Each self-attention and FFN branch of block $b$ has its own learnable
per-channel scale, $\alpha^{(b)}_{\mathrm{sa}}\in\mathbb R^d$ or
$\alpha^{(b)}_{\mathrm{ffn}}\in\mathbb R^d$. For example, the self-attention
branch maps its input $X$ to
$X+\alpha^{(b)}_{\mathrm{sa}}\odot\mathrm{SA}\bigl(\mathrm{RMSNorm}(X)\bigr)$,
where $\mathrm{SA}$ is the cosine self-attention sublayer and the product acts
on the channels of every token. Both scales are initialized to the constant
$1/B_\ell$. A head of $B_\ell$ blocks therefore starts as a small perturbation
of the identity whose size is set by its depth, and
Sec.~\ref{app:lipschitz} shows that this initialization keeps the bound of
Eq.~(\plaineqref{eq:lip}) below a constant that does not depend on the depth.
The cross-attention branch has the same form with a gate $\alpha_{\mathrm{ca}}$
in place of the residual scale. The gate is initialized to zero, so conditioning
contributes nothing at initialization. The output projection of the
cross-attention branch is initialized normally (Sec.~\ref{app:hparams}), so
the gate receives a nonzero gradient from the first update.

\noindent\textbf{Logit head.}
Every head ends in a final RMSNorm and a linear logit head, which applies one
spectrally normalized linear map to each token independently. Each patch token
and the summary token therefore yields one logit. Level $\ell$ thus emits
$N_\ell$ logits, and concatenating the $L$ levels gives the
$n=\sum_\ell N_\ell$ logits of Eq.~(\plaineqref{eq:dense}). 
We also
implemented a convolutional logit head. It reshapes the patch tokens to their
spatial grid and applies consecutive spectrally normalized $4\times4$ grouped
convolutions, which reduce the grid to a $5\times5$ logit map. In this head,
the summary token can contribute one further logit through its own linear map.
In our experiments, the two logit heads performed comparably. We therefore
report only the linear head, which is also the head analyzed in
Sec.~\ref{app:lipschitz}.
}

\subsection{Dataset and evaluation protocols}
\label{app:protocol}

{\color{revtext}%
\noindent\textbf{Real-image training set.}
\scopechange{We use the fixed
\href{https://huggingface.co/tianweiy/DMD2/tree/main/data/laion_vae_latents}{dataset of 500K real image--caption pairs} constructed by
DMD2~\citep{yin2024dmd2} from LAION-Aesthetics
5.5+~\citep{schuhmann2022laion5b} for both SDXL and SD3.5-M. The dataset
discards unsafe images and images whose width or height is below 1024 pixels.
Model-specific decoding, resizing, and feature extraction are applied only
after sampling from this common pool.}

\noindent\textbf{COCO evaluation sets.}
\scopechange{For both models, we follow the DMD2 SDXL protocol without
resampling its prompts or reference images. We use the fixed draw of 10K
distinct images from the \href{https://cocodataset.org/}{COCO-2014 validation
set} and one associated caption per image, as stored in the
\href{https://huggingface.co/tianweiy/DMD2/blob/main/data/coco/coco10k.zip}{\texttt{coco10k/all\_prompts.pkl} file and paired
\texttt{coco10k/subset} directory}. We generate one $1024\times1024$ image per
prompt. The prompt order and sampling seeds are fixed and shared by all
methods.}

\noindent\textbf{FID and CLIP score.}
\scopechange{We compute FID with the DMD2 implementation
of \texttt{clean-fid} in its clean mode. Before Inception feature extraction,
the longer spatial dimension is center-cropped to form a square and the image
is resized to $512\times512$ with Lanczos interpolation. The resulting FID
compares 10K generated images with their corresponding 10K real images. We
compute CLIP score with the DMD2 evaluator for both
models. It center-crops each generated image, resizes it to $256\times256$, and
applies the OpenCLIP ViT-g/14 model pretrained on LAION-2B
(\texttt{laion2b\_s12b\_b42k}). The text encoder receives the original caption
prefixed by ``A photo depicts''. We report the mean cosine similarity over all
generated image--caption pairs.}

\noindent\textbf{Auxiliary metrics.}
All auxiliary COCO metrics use the same generated images and the unmodified
captions used by the corresponding FID evaluation. We report the mean raw
score of LAION-Aesthetics Predictor V2, the
official HPSv2 score~\citep{wu2023hpsv2}, and the score produced by
\texttt{yuvalkirstain/PickScore\_v1}~\citep{kirstain2023pickscore}. HPSv2 and
PickScore are shown on the $\times100$ scale used in the tables. ImageReward is
the mean raw score from \texttt{ImageReward-v1.0}
~\citep{xu2023imagereward}. GenEval~\citep{ghosh2023geneval}, reported only for
SD3.5-M, is evaluated separately because it defines its own prompt suite. We
use the released 553 prompts, generate four images per prompt, run the official
Mask2Former/OpenCLIP evaluator with its default thresholds, and report the
unweighted mean accuracy over its six compositional tasks.

\noindent\textbf{Checkpoint selection.}
We evaluate ROAD checkpoints every 500 training steps and select the checkpoint
with the lowest FID. Every other metric in a row is computed from that same
checkpoint; we never select a separate checkpoint for alignment, aesthetics,
human preference, or GenEval. Public baselines are evaluated with the same
prompt files, reference images, preprocessing, metric implementations, and
reporting scales as the corresponding ROAD model.
}

\subsection{Training hyperparameters}
\label{app:hparams}

\scopechange{Tab.~\ref{tab:hparams} lists the training configurations of SDXL
and SD3.5-M.} 
On SDXL the critic scores the summary (cls) token
alongside the patch tokens, the register tokens of DINOv3 are
dropped, the cross-attention gate is zero-initialized while its
output projection is normally initialized, and the critic and student
are updated at a 1:1 ratio. Each of the four projected feature levels
receives its own head, with depths decreasing from shallow to deep,
$B_\ell{=}(4,3,2,1)$: shallower backbone representations receive more
trainable reorganization steps, while already highly semantic deep
representations need only light adaptation; every head ends in a
token-wise linear logit map (block structure in
Sec.~\ref{app:lipschitz}). SD3.5-M uses the same head design as
SDXL.

\begin{table}[h]
\centering
\caption{\scopechange{\textbf{Training configurations of SDXL and SD3.5-M.}}}
\label{tab:hparams}
\begin{minipage}{0.88\linewidth}
\centering
\footnotesize
\setlength{\tabcolsep}{4pt}
\renewcommand{\arraystretch}{1.1}
\begin{tabular}{@{}>{\raggedright\arraybackslash}p{0.30\linewidth}
                    >{\raggedright\arraybackslash}p{0.32\linewidth}
                    >{\raggedright\arraybackslash}p{0.32\linewidth}@{}}
\toprule
\scopechange{Setting} & \scopechange{SDXL} & \scopechange{SD3.5-M} \\
\midrule
Teacher model & SDXL-Base-1.0 & SD3.5-Medium \\
Teacher arch. & UNet & MMDiT \\
\mbox{Conditioning timestep / $\sigma$} & $t{=}999$ & $\sigma{=}1.0$ \\
GPUs $\times$ batch & $8{\times}8 = 64$ & \scopechange{$8{\times}4 = 32$} \\
G learning rate & $5{\times}10^{-7}$ & $5{\times}10^{-7}$ \\
D learning rate & $5{\times}10^{-6}$ & $5{\times}10^{-6}$ \\
Optimizer & AdamW, $\beta{=}(0,0.99)$, no weight decay & AdamW, $\beta{=}(0,0.99)$, no weight decay \\
\mbox{Gradient clipping (G / D)} & 10 / 10 & \scopechange{1.0 / 10.0, global $L_2$ norm} \\
Position / linear norm & RoPE / SN & \scopechange{RoPE / SN} \\
Refresh interval $\tau$ & 5k & 5k \\
Gradient penalty & \scopechange{$\gamma_{\mathrm{J}}{=}7.5{\times}10^{-3}$ with one Rademacher probe; lazy interval $k_{\mathrm{lazy}}{=}16$}. & \scopechange{$\gamma_{\mathrm{J}}{=}7.5{\times}10^{-4}$ with one Rademacher probe; lazy interval $k_{\mathrm{lazy}}{=}16$}. \\
\bottomrule
\end{tabular}
\end{minipage}
\end{table}

\needspace{6\baselineskip}
\section{Qualitative Results}
\label{app:qualitative}

\subsection{Prompts for Qualitative Figures}
\label{app:qualitative-prompts}
\begingroup\small
We use the following prompts for Fig.~\ref{fig:teaser}, ordered by the top-left corner of each image, from top to bottom and then left to right:
\begin{itemize}
  \item A steampunk woman wearing brass goggles stands by a sunlit window in an ornate Victorian room, beside a large vintage mechanical instrument.
  \item bedroom melting into the galaxy, penthouse apartment bedroom with big windows, unsplash, photo of the day
  \item mystical fantasy cat of many colors
  \item Ship in the form of a fish tank with fish and corals
  \item crowd of people on top of a skyscraper at night standing near a spotlight cityscape around them, from high angle looking down on rooftop, bladerunner
  \item Photo English country palace of aristocrats on a background of nature painted in watercolor on paper
  \item a unicorn standing in a meadow, morning sun in the background.
  \item concept art of a handsome and regal looking evil sorcerer with a circlet of thorns on his head, magic and lightning
  \item A glass teapot filled with blooming jasmine tea on a wooden table by a rainy window.
  \item A teddy bear reading a newspaper on a park bench.
\end{itemize}
\endgroup

The prompts for the SDXL comparison in Fig.~\ref{fig:road-sdxl-comparison}, from top to bottom, are:
\begin{itemize}
  \item isometric beach scenery
  \item Russian-inspired building with towers and spires, smokestacks spewing smoke in the background against a red sky.
  \item A cute mouse with a dress in a kitchen, fantasy, cinematic, door
  \item An old man playing chess against a robot.
  \item steampunk chameleon automaton bot, close up, realistic, highly detail
  \item An oil painting of a young woman dressed in a flowing, white dress sitting on a chair with a red drape. An easel rests on her knees and she is evidently drawing. She is gazing directly at the observer.
\end{itemize}

The prompts for the SD3.5-M comparison in Fig.~\ref{fig:road-sd35m-comparison}, from top to bottom, are:
\begin{itemize}
  \item an anthropomorphic owl courier, monochromatic with subtle red accents
  \item A futuristic city with flying cars, tall glass towers covered in plants, and a large park in the center, bright daylight.
  \item A bookshelf filled with colorful books, a potted plant and a small globe.
  \item A photo of an cosplay girl dressed as a fantasy warrior, with a full suit of armor and a sword at her side, standing on a hilltop with a dramatic sunset behind her
  \item bedroom melting into the galaxy, penthouse apartment bedroom with big windows, unsplash, photo of the day
  \item A teddy bear reading a newspaper on a park bench.
\end{itemize}

\clearpage
\subsection{Qualitative Samples from ROAD}
\label{app:qualitative-samples}
Figures~\ref{fig:road-sdxl-samples} and~\ref{fig:road-sd35m-samples} show
qualitative generations from ROAD distilled on SDXL and SD3.5-M, respectively.

\begin{center}
\begin{minipage}{\linewidth}
\captionsetup{hypcap=false}
  \centering
  \includegraphics[width=\textwidth,height=0.79\textheight,keepaspectratio]{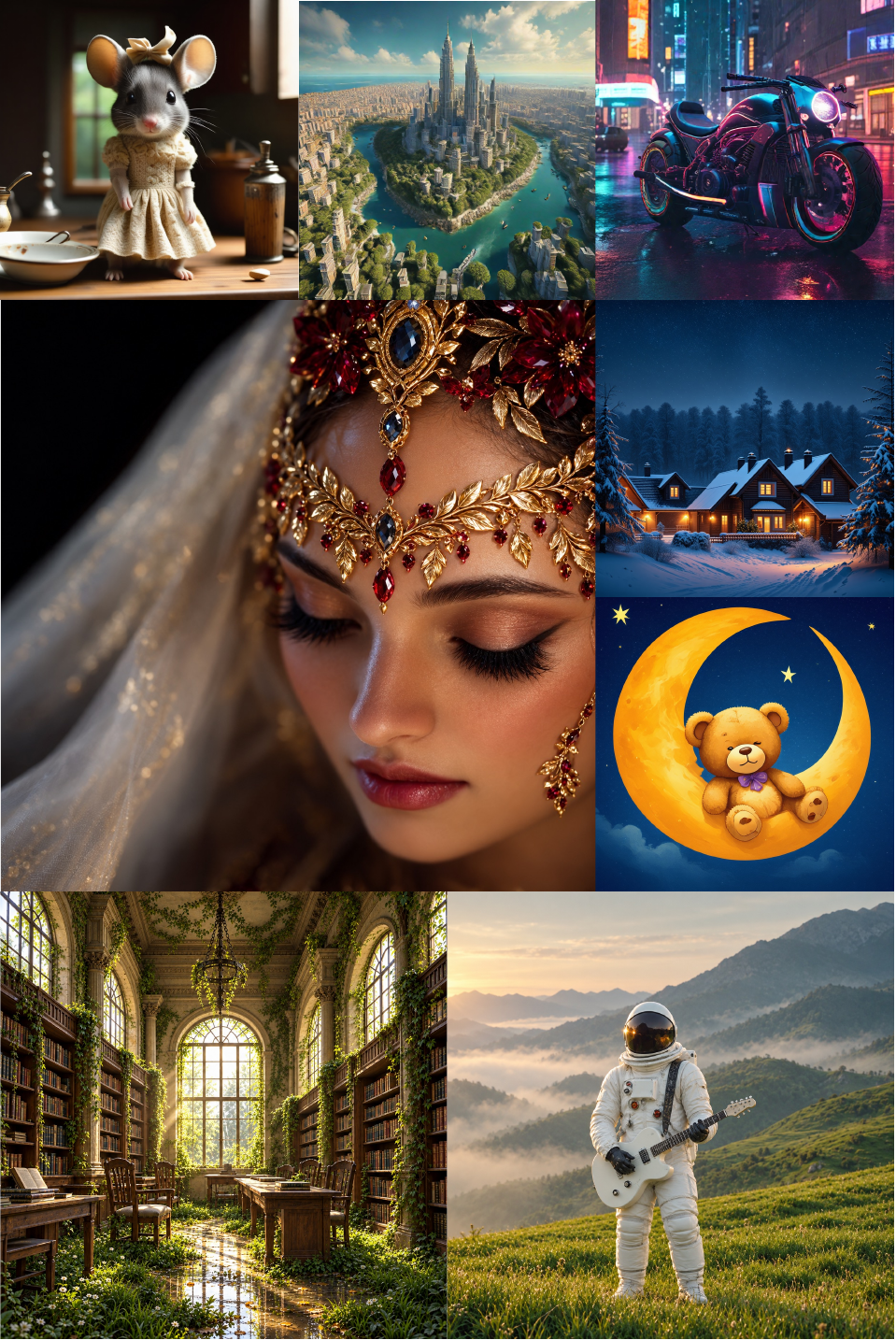}
  \captionof{figure}{\textbf{ROAD samples on SDXL.} Qualitative generations from the distilled model.}
  \label{fig:road-sdxl-samples}
\end{minipage}
\end{center}
\clearpage
\begin{figure}[!p]
  \centering
  \includegraphics[width=\textwidth,height=0.9\textheight,keepaspectratio]{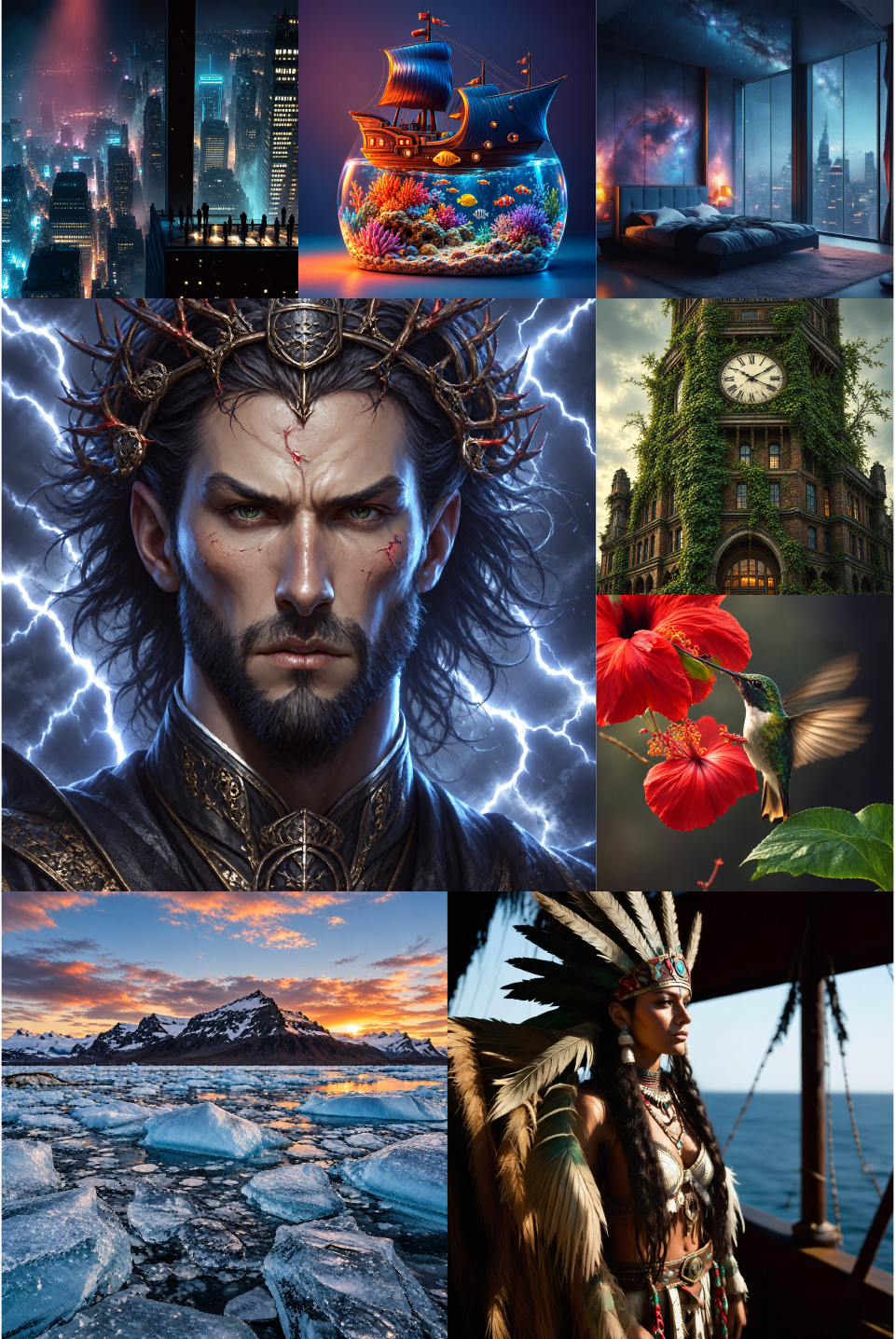}
  \caption{\textbf{ROAD samples on SD3.5-M.} Qualitative generations from the distilled model.}
  \label{fig:road-sd35m-samples}
\end{figure}
\FloatBarrier

\newlength{\RoadSlotWidth}
\newcommand{\RoadComparisonRow}[2]{%
  \vbox{%
    \vskip 2pt
    \hbox to \linewidth{\parbox[c]{\linewidth}{\centering\fontsize{8.5}{10.5}\selectfont #1}}%
    \vskip 2pt
    \nointerlineskip
    \hbox to \linewidth{\includegraphics[width=\linewidth]{#2}}%
  }\relax
}
\newcommand{\RoadMethodLabel}[1]{%
  \parbox[t]{\RoadSlotWidth}{\centering\footnotesize #1}%
}
\newcommand{\RoadComparisonGrid}[6]{%
  \par
  \begingroup
  \setlength{\RoadSlotWidth}{\dimexpr\linewidth/5-8pt/5\relax}%
  #1%
  \hbox to \linewidth{%
    \RoadMethodLabel{#2}\hfil
    \RoadMethodLabel{#3}\hfil
    \RoadMethodLabel{#4}\hfil
    \RoadMethodLabel{#5}\hfil
    \RoadMethodLabel{#6}}%
  \par
  \endgroup
}

\subsection{Qualitative Comparison with Other Distillation Methods}
\label{app:qualitative-comparison}

\begin{center}
\begin{minipage}{0.90\textwidth}
\captionsetup{hypcap=false,skip=5pt}
\RoadComparisonGrid{%
  \RoadComparisonRow{isometric beach scenery}{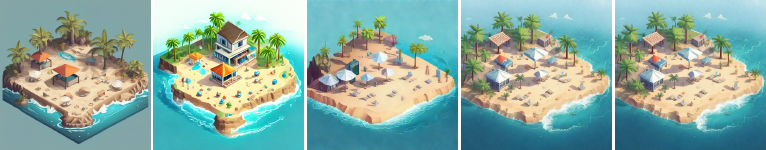}
  \RoadComparisonRow{Russian-inspired building with towers and spires, smokestacks spewing smoke in the background against a red sky.}{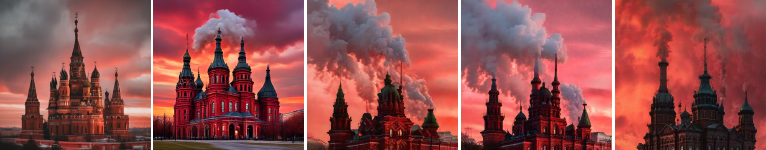}
  \RoadComparisonRow{A cute mouse with a dress in a kitchen, fantasy, cinematic, door}{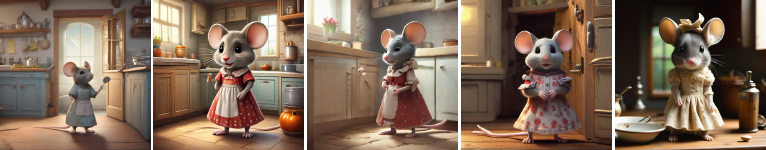}
  \RoadComparisonRow{An old man playing chess against a robot.}{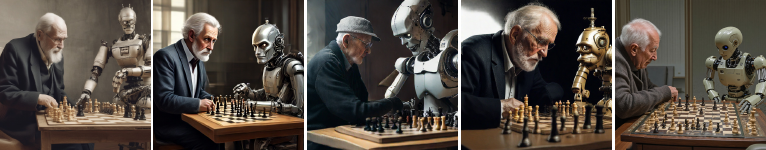}
  \RoadComparisonRow{steampunk chameleon automaton bot, close up, realistic, highly detail}{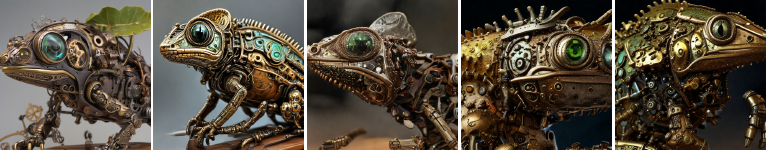}
  \RoadComparisonRow{An oil painting of a young woman dressed in a flowing, white dress sitting on a chair with a red drape. An easel rests on her knees and she is evidently drawing. She is gazing directly at the observer.}{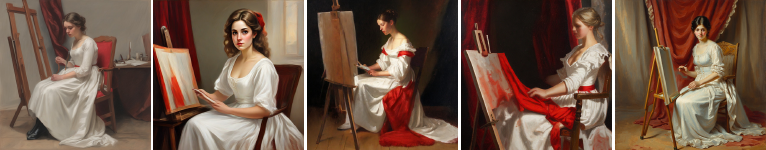}
}{SDXL}{Hyper-SDXL}{DMD2}{NitroSD-Realism}{ROAD}
\captionof{figure}{\textbf{Qualitative SDXL comparison under shared prompts.} The SDXL teacher uses 40 inference steps; all four distilled methods use one step.}
\label{fig:road-sdxl-comparison}
\end{minipage}
\end{center}

\clearpage
\begin{center}
\begin{minipage}{0.90\textwidth}
\captionsetup{hypcap=false,skip=5pt}
\RoadComparisonGrid{%
  \RoadComparisonRow{an anthropomorphic owl courier, monochromatic with subtle red accents}{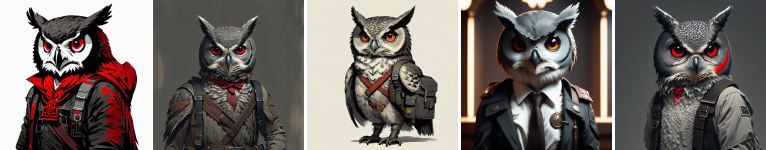}
  \RoadComparisonRow{A futuristic city with flying cars, tall glass towers covered in plants, and a large park in the center, bright daylight.}{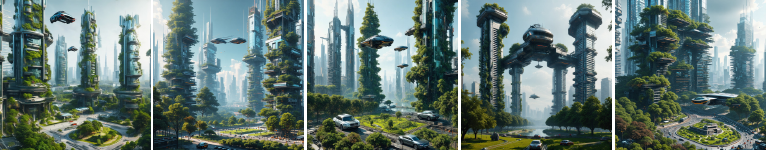}
  \RoadComparisonRow{A bookshelf filled with colorful books, a potted plant and a small globe.}{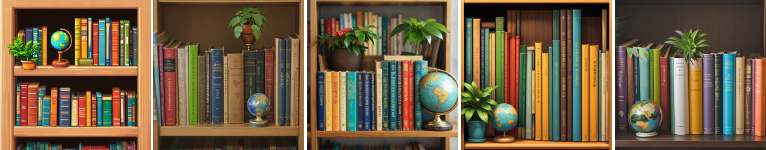}
  \RoadComparisonRow{A photo of an cosplay girl dressed as a fantasy warrior, with a full suit of armor and a sword at her side, standing on a hilltop with a dramatic sunset behind her}{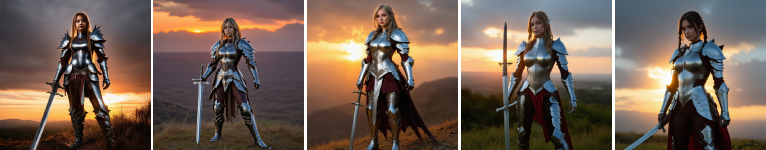}
  \RoadComparisonRow{bedroom melting into the galaxy, penthouse apartment bedroom with big windows, unsplash, photo of the day}{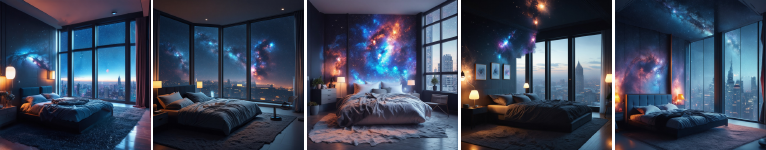}
  \RoadComparisonRow{A teddy bear reading a newspaper on a park bench.}{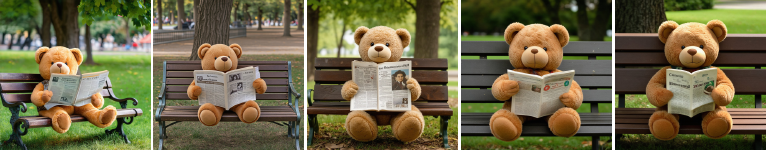}
}{SD3.5-M\\{\scriptsize 40 steps / 80 NFE}}{SiD-DiT\\{\scriptsize 4 steps / 4 NFE}}{SwD\\{\scriptsize 4 steps / 4 NFE}}{SenseFlow\\{\scriptsize 4 steps / 4 NFE}}{ROAD\\{\scriptsize 1 step / 1 NFE}}
\captionof{figure}{\textbf{Qualitative SD3.5-M comparison under shared prompts.} Inference steps and network evaluations (NFE) are given below the columns.}
\label{fig:road-sd35m-comparison}
\end{minipage}
\end{center}
\clearpage

\end{document}